\documentclass{article}

\usepackage{iclr2026_conference,times}

\usepackage{amsmath,amssymb,amsthm}
\usepackage{booktabs}
\usepackage{graphicx}
\usepackage{colortbl}
\usepackage{multirow}
\usepackage{algorithm}
\usepackage{algpseudocode}
\usepackage{tikz}
\usetikzlibrary{arrows.meta,positioning,shapes.misc,decorations.pathmorphing}

\newif\ifrevisions
\revisionsfalse
\definecolor{editblue}{RGB}{0,80,180}
\definecolor{linknavy}{RGB}{0,60,130}
\newcommand{\mnote}[1]{\ifrevisions{\color{editblue}#1}\else#1\fi}

\definecolor{cellbest}{RGB}{198,239,206}
\definecolor{cellworst}{RGB}{255,199,206}
\definecolor{heatpos}{RGB}{99,190,123}
\definecolor{heatneg}{RGB}{248,105,107}
\definecolor{leanbadge}{RGB}{187,214,190}

\newcommand{\leantag}{The proof is mechanized in Lean~4 (\autoref{app:certification}).}

\usepackage{hyperref}
\usepackage{url}
\hypersetup{
  colorlinks=true,
  linkcolor=linknavy,   %
  citecolor=linknavy,   %
  urlcolor=linknavy,
  breaklinks=true,
}

\newtheorem{theorem}{Theorem}
\newtheorem{proposition}{Proposition}

\newtheorem{corollary}{Corollary}
\theoremstyle{definition}
\newtheorem{definition}{Definition}
\newtheorem{assumption}{Assumption}
\theoremstyle{remark}
\newtheorem{remark}{Remark}

\newcommand{\E}{\mathbb{E}}
\newcommand{\Prob}{\mathbb{P}}

\newcommand{\Otil}{\widetilde{O}}
\newcommand{\eps}{\varepsilon}
\newcommand{\spread}{\mathrm{spread}}
\newcommand{\bias}{\mathrm{bias}}                %
\newcommand{\leaves}[1]{\mathcal{L}(#1)}     %
\newcommand{\val}{f}                          %
\newcommand{\muhat}{\widehat{\mu}}
\newcommand{\Hedge}{H_{\mathrm{edge}}}
\newcommand{\Hblind}{H_{\mathrm{blind}}}
\newcommand{\cp}{c_{p}}                        %
\newcommand{\cl}{c_{\ell}}                     %
\newcommand{\nodim}{d}                         %

\newcommand{\pending}[1][run]{\textbf{[pending: #1]}}
\newcommand{\figorpending}[3][\linewidth]{%
  \IfFileExists{#2}{\includegraphics[width=#1]{#2}}%
  {\fbox{\parbox[c][0.28\textheight][c]{#1}{\centering\pending[figure from \texttt{#3}]}}}}
\newcommand{\tableorpending}[2]{%
  \IfFileExists{#1}{\input{#1}}{\pending[table from \texttt{#2}]}}

\title{Canopy: Exploiting Piecewise Smooth Tree Priors for Multi-Fidelity Bandits}

\iclrfinalcopy
\author{%
  Michael Jerge \\
  Amazon \\
  \texttt{mjerge@amazon.com} %
  \And
  Suman Jana \\
  Columbia University \\
  \texttt{sj2754@columbia.edu} %
}

\begin{document}
\maketitle
\lhead{Preprint. Under review.}

\begin{abstract}

Many LLM inference problems, including model routing, prefix-cache
management, prompt trimming, and test-time search, can be viewed as
optimization over a tree. This structure arises naturally from autoregressive generation: every prefix defines a node, and its continuations form a subtree below it. Internal nodes of the tree provide cheap but biased estimates of a region’s value, while leaf evaluations are expensive
but accurate. Hierarchical bandit methods can exploit this structure, but typically require a specific smoothness schedule to be specified in
advance, even though real objectives are often only piecewise smooth and their optima may lie near sharp boundaries. We introduce CANOPY, a multi-fidelity tree bandit that learns where the smoothness prior is valid rather than assuming it globally. CANOPY uses cheap random-path probes to construct an online certificate of local aggregation bias, then directs expensive leaf evaluations toward cells where the certificate detects a smoothness violation. We prove fixed-budget and regret guarantees whose additional cost is additive
in the number of discontinuities, recovering the smooth-tree rate when no violations are present and approaching structure-blind search as
violations become dense. Across routing, top-(k) identification, test-time search, caching, and prompt trimming, CANOPY consistently
improves matched-budget performance, including $2.9\times$ higher top-$10$ recall on a $1000$-model pool, $1.6\times$ more SWE-bench Verified issues resolved than best-of-(N), and $3.6\times$ lower median time-to-first-token
with prefix caching.

\end{abstract}

\section{Introduction}
\label{sec:intro}
Modern LLM applications spend an increasing fraction of their compute at inference time~\citep{welleck2024metageneration,snell2024scaling,brown2024monkeys}, and the algorithms that use this compute often revisit shared partial computations. For example, search procedures expand multiple traces from a common prefix~\citep{yao2023tot}, serving stacks reuse cached prefixes across requests~\citep{kwon2023vllm,zheng2024sglang}, and prompt-trimming methods consider progressively compressed versions of a long prompt~\citep{jiang2023llmlingua}. In each setting, the system must choose among more candidates than it can afford to evaluate exhaustively, since evaluating a candidate may require a generation, a test run, or a served request. These candidates admit a natural hierarchical organization: candidates that share more of their computation tend to behave more similarly, except at a small number of sharp boundaries.

This hierarchy induces a multi-fidelity feedback model (\hyperref[fig:overview]{Figure~\ref*{fig:overview}a}). An internal-node probe returns the subtree average, giving a cheap but biased proxy for the best leaf beneath it, while a leaf evaluation is expensive but unbiased. The key quantity is therefore the \emph{aggregation bias}: how far the best leaf in a cell can exceed its mean. Classical hierarchical bandits control this bias using a smoothness schedule specified in advance~\citep{bubeck2011xarmed,munos2011soo,azar2014hct}. But the right schedule is difficult to know: assuming too much smoothness can prune promising regions, while being too conservative forfeits the savings from hierarchical search. Moreover, real objectives may be only piecewise smooth, with a small number of sharp boundaries where an innocuous subtree average hides an unusually valuable leaf. This raises our central question: can the learner use cheap observations to determine where the smoothness assumption is actually valid?

\begin{figure}[t]
\centering
\resizebox{\textwidth}{!}{%
\begin{tikzpicture}[
  inode/.style={circle, draw=black!60, fill=black!10, minimum size=10.5mm, inner sep=0pt,
    font=\LARGE},
  probe/.style={circle, draw=blue!70!black, fill=blue!20, minimum size=10.5mm, inner sep=0pt,
    thick, font=\LARGE},
  leaf/.style={rectangle, draw=black!60, fill=black!5, minimum size=7.5mm, inner sep=0pt,
    font=\LARGE},
  leafeval/.style={rectangle, draw=orange!80!black, fill=orange!30, minimum size=7.5mm,
    inner sep=0pt, thick, font=\LARGE},
  gnode/.style={circle, draw=black!30, fill=black!5, minimum size=10.5mm, inner sep=0pt,
    font=\LARGE, text=black!40},
  gleaf/.style={rectangle, draw=black!30, fill=black!3, minimum size=7.5mm, inner sep=0pt,
    font=\LARGE, text=black!40},
  edge/.style={black!45},
  gedge/.style={black!20},
  lbl/.style={font=\LARGE\bfseries},
  slbl/.style={font=\LARGE}]

\begin{scope}[xshift=0cm]
  \node[lbl, anchor=west] at (-0.6, 5.1) {(a) Multi-fidelity tree};
  \node[inode] (r) at (3.5, 4) {.53};
  \node[probe] (a1) at (1.5, 2.7) {.60};
  \node[inode] (a2) at (5.5, 2.7) {.45};
  \foreach \i/\x/\v in {1/0.5/.55, 2/2.2/.65, 4/4.2/.4, 5/5.55/.55, 6/6.9/.3}{
    \node[inode, minimum size=9.5mm] (b\i) at (\x, 1.5) {\v};}
  \foreach \i/\x/\v in {1/0.0/.5, 2/0.9/.6, 3/1.8/.6, 4/2.7/.7, 5/4.2/.4, 7/6.0/.2,
                        8/6.9/.3}{
    \node[leaf] (l\i) at (\x, 0.3) {\v};}
  \node[leafeval] (l6) at (5.1, 0.3) {.9};
  \draw[edge] (r)--(a1); \draw[edge] (r)--(a2);
  \draw[edge] (a1)--(b1); \draw[edge] (a1)--(b2);
  \draw[edge] (a2)--(b4); \draw[edge] (a2)--(b5); \draw[edge] (a2)--(b6);
  \draw[edge] (b1)--(l1); \draw[edge] (b1)--(l2); \draw[edge] (b2)--(l3);
  \draw[edge] (b2)--(l4); \draw[edge] (b4)--(l5); \draw[edge] (b5)--(l6);
  \draw[edge] (b5)--(l7); \draw[edge] (b6)--(l8);
  \draw[dashed, red!70!black, rounded corners=2pt] (4.6, -0.2) rectangle (6.5, 0.8);
  \draw[red!70!black, very thick, decorate,
        decoration={zigzag, segment length=2.2mm, amplitude=0.7mm}]
    (5.55, -0.15) -- (5.55, 0.75);
  \node[slbl, blue!70!black, align=center] at (0.4, 3.85)
    {cheap probe $\cp$\\(biased average)};
  \node[slbl, orange!80!black, align=center] at (2.6, -0.95)
    {expensive eval $\cl$\\(unbiased)};
  \node[slbl, red!70!black, anchor=west, align=left] at (6.4, -0.95)
    {Lipschitz\\boundary};
  \draw[red!70!black, thin] (6.45, -0.6) -- (5.65, -0.2);
\end{scope}

\begin{scope}[xshift=10.5cm]
  \node[lbl, anchor=west] at (-0.6, 5.1) {(b) Data-driven certificate};
  \foreach \x/\h/\c in {0.6/0.22/cellbest, 3.1/0.22/cellbest, 5.6/1.3/cellworst}{
    \draw[rounded corners=2pt, fill=\c, draw=black!40] (\x-0.6, 0) rectangle (\x+0.6, 1.7);
    \draw[fill=black!55, draw=none] (\x-0.15, 0.2) rectangle (\x+0.15, 0.2+\h);
  }
  \draw[dashed, red!70!black, thick] (-0.4, 0.62) -- (6.6, 0.62);
  \node[slbl, red!70!black, anchor=west] at (6.7, 0.62) {floor $L\rho^{\ell}$};
  \node[slbl] at (0.6, 2.1) {$\{.5,\,.6\}$};
  \node[slbl] at (3.1, 2.1) {$\{.6,\,.7\}$};
  \node[slbl, red!70!black] at (5.6, 2.1) {$\{.9,\,.2\}$};
  \node[slbl] at (1.85, 2.85) {certified smooth};
  \node[slbl, red!70!black] at (5.6, 2.85) {flagged};
  \node[slbl, align=center] at (3.1, -0.95)
    {bar $=$ within-cell spread\\from cheap random-path probes};
\end{scope}

\begin{scope}[xshift=21cm]
  \node[lbl, anchor=west] at (-0.6, 5.1) {(c) Discontinuity-guided sampling};
  \node[inode] (r2) at (3.5, 4) {.53};
  \node[gnode] (c1) at (1.5, 2.7) {.60};
  \node[inode] (c2) at (5.5, 2.7) {.45};
  \foreach \i/\x/\v in {1/0.5/.55, 2/2.2/.65}{
    \node[gnode, minimum size=9.5mm] (gb\i) at (\x, 1.5) {\v};}
  \foreach \i/\x/\v in {1/0.0/.5, 2/0.9/.6, 3/1.8/.6, 4/2.7/.7}{
    \node[gleaf] (gl\i) at (\x, 0.3) {\v};}
  \foreach \i/\x/\v in {4/4.2/.4, 5/5.55/.55, 6/6.9/.3}{
    \node[inode, minimum size=9.5mm] (d\i) at (\x, 1.5) {\v};}
  \node[leaf] (m5) at (4.2, 0.3) {.4};
  \node[leaf] (m8) at (6.9, 0.3) {.3};
  \draw[rounded corners=2pt, fill=cellworst, draw=black!40] (4.55, -0.2) rectangle (6.55, 1.0);
  \draw[red!70!black, very thick, decorate,
        decoration={zigzag, segment length=2.2mm, amplitude=0.7mm}]
    (5.55, -0.15) -- (5.55, 0.75);
  \node[leafeval] (m6) at (5.1, 0.3) {.9};
  \node[leafeval] (m7) at (6.0, 0.3) {.2};
  \draw[gedge] (r2)--(c1); \draw[edge] (r2)--(c2);
  \draw[gedge] (c1)--(gb1); \draw[gedge] (c1)--(gb2);
  \draw[gedge] (gb1)--(gl1); \draw[gedge] (gb1)--(gl2);
  \draw[gedge] (gb2)--(gl3); \draw[gedge] (gb2)--(gl4);
  \draw[edge] (c2)--(d4); \draw[edge] (c2)--(d5); \draw[edge] (c2)--(d6);
  \draw[edge] (d4)--(m5); \draw[edge] (d6)--(m8);
  \draw[edge] (d5)--(m6); \draw[edge] (d5)--(m7);
  \foreach \x in {5.1, 6.0}{
    \draw[-{Stealth[length=2mm]}, orange!80!black, thick] (\x, 1.35) -- (\x, 0.75);}
  \node[slbl, black!45, align=center] at (0.4, 3.85)
    {certified smooth:\\pruned by probes};
  \node[slbl, align=center] at (3.5, -0.95)
    {dense evals on the flagged cell\\recover $x^\star{=}.9$};
\end{scope}
\end{tikzpicture}}
\caption{\textbf{Method overview.} (a)~Leaves show the unknown values, and the red line
marks the Lipschitz boundary splitting the cell.
(b)~The certificate measures each cell's within-cell spread from cheap probes against the Lipschitz floor.
(c)~On the same tree, the search prunes the certified-smooth subtree.}
\label{fig:overview}
\end{figure}
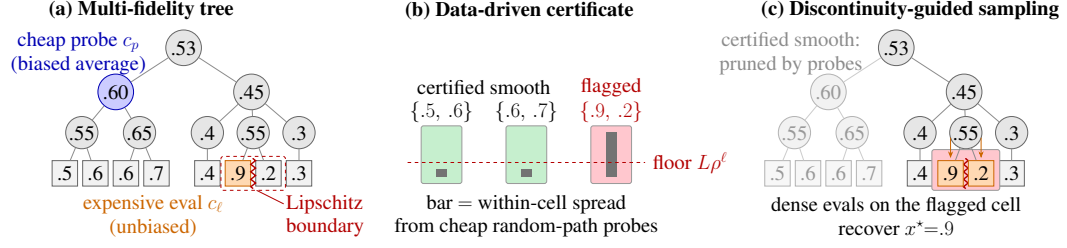

\paragraph{Contributions.}
We study a multi-fidelity bandit over a complete tree without access to the smoothness schedule that standard hierarchical bandits assume. Internal nodes provide cheap subtree averages under a cost budget, while leaves provide expensive unbiased evaluations. We use the standard hierarchical descent of $\mathcal{X}$-armed bandits~\citep{bubeck2011xarmed,azar2014hct,munos2011soo,grill2015poo}, but, to our knowledge, are the first to certify the required smoothness information directly from data rather than assume it or search over candidate schedules. We further show that once a smoothness violation is encountered, targeted leaf-level sampling is unavoidable, up to logarithmic factors, for identifying the best leaf (\autoref{prop:lower}).

We make three key contributions. First, we develop an online, local-Lipschitz certificate of aggregation bias (\autoref{sec:smoothness}). Second, we build a discontinuity-guided adaptive sampling technique that spends expensive evaluations only in regions where the certificate fails (\autoref{sec:theory}). Third, we combine the two into a cost-budgeted framework (\autoref{sec:setup}) that casts model routing, prefix caching, and test-time search as instances of the same multi-fidelity tree optimization problem.

\paragraph{Experimental overview.}
Across twelve open-source benchmarks, our method improves performance at matched cost or compute over others (\autoref{tab:summary}). 
On RouterEval's $1000$-model pool, the tree bandit reaches $2.9\times$ the top-$10$ recall of the strongest structure-blind baseline 
and is better at every budget. Value-guided search beats best-of-N by $+0.087$ on MATH,
$+0.111$ on GPQA-Diamond, and $+0.123$ on SWE-bench Verified, with the strict budget cells
of the model$\times$budget sweep reaching $1.8$--$2.2\times$. In a serving-based setting, the regional router beats the flat learner in $4$ of $5$
independent $\tau$-bench replicates, and reaches near-oracle quality at the lowest cost on
live MMLU. On the Mooncake production trace, the adaptive cache reuses $78\%$ more blocks per
request than the engine-default LRU after a popularity shift. On a vLLM server, prefix
caching cuts median time-to-first-token $3.6\times$, and adaptive eviction saves $83\%$ more
TTFT than LRU post-shift. The experiments also identify precisely when these gains are available and measure every boundary condition
(\autoref{sec:experiments}).

\section{Background}
\label{sec:background}

\paragraph{Hierarchical and $\mathcal{X}$-armed bandits.}
Optimizing a noisy function over a hierarchical partition of its domain is the classical
$\mathcal{X}$-armed bandit problem~\citep{bubeck2011xarmed,munos2014optimistic}.
Optimistic tree-search methods differ primarily in how they control the gap between a
cell's observed value and its best descendant. HOO~\citep{bubeck2011xarmed} uses a
known smoothness schedule, while SOO/StoSOO~\citep{munos2011soo,valko2013stosoo}
avoid specifying one explicitly and POO~\citep{grill2015poo} searches over a family of
candidate schedules. Adaptation to unknown smoothness is impossible from leaf
evaluations alone~\citep{locatelli2018adaptivity}. Our setting provides additional
feedback: a random-path probe samples the distribution of leaf values within a cell,
allowing us to certify the corresponding maximum-minus-mean aggregation bias directly
from data (\autoref{sec:smoothness}).

\paragraph{Multi-fidelity optimization and best-arm identification.}
Multi-fidelity optimization combines cheap biased observations with expensive accurate
ones~\citep{kandasamy2016multifidelity}. Our setting instantiates this structure on a
tree: internal-node probes return cheap subtree averages, while leaf evaluations provide
expensive unbiased observations. Since a subtree average can underestimate its best
leaf, the relevant fidelity bias is one-sided and is exactly the aggregation bias studied
in \autoref{sec:smoothness}. Fixed-budget best-arm identification provides the
corresponding leaf-level problem, where a problem-dependent hardness quantity determines
how quickly identification error decreases with budget~\citep{audibert2010bestarm,
kaufmann2016complexity}. Our method combines these two regimes: cheap structural search
over regions where aggregation is reliable, and targeted leaf-level certification where it
is not.

\section{Theory}
\label{sec:theory-main}
We define the tree model of \autoref{sec:setup} and the optimistic descent of \autoref{sec:algorithm}. Both are organized around how far a cell's best leaf can exceed its mean. The model makes that quantity the bias of a cheap probe, and the descent carries it as its optimism bonus. The guarantees of \autoref{sec:theory} are stated in terms of this gap, which the certificate of \autoref{sec:smoothness} bounds from data.

\subsection{Notation and Preliminaries}
\label{sec:setup}

\paragraph{The structure.}
We study a complete $b$-ary tree of depth $D$, with $N=b^{D}$ leaves. A node $v$ at level $\ell\in\{0,\dots,D\}$ spans a contiguous block of $|\leaves{v}|=b^{D-\ell}$ leaves. Each leaf $x$ has a true mean reward
$\mu(x)\in[0,1]$, and the value of any node is the average of its subtree's leaves,
\begin{equation}
  \val(v) \;=\; \frac{1}{|\leaves{v}|}\sum_{x\in\leaves{v}}\mu(x),
  \qquad\text{so}\qquad \val(v)=\frac{1}{b}\sum_{c\in\mathrm{children}(v)}\val(c).
  \label{eq:avg-backup}
\end{equation}
We denote $\mu^\star=\max_x\mu(x)$ for the best leaf. The aggregation bias, $\bias(v)$, of an internal node is the gap between the largest leaf mean in its subtree, $\max_{x\in\leaves{v}}\mu(x)$, and the average $\val(v)$ returned.

We measure the distance between leaves by their lowest common ancestor (LCA), $d(x,y)=\rho^{\mathrm{level}(\mathrm{LCA}(x,y))}$ with $\rho\in(0,1)$. Two leaves are considered close when they share a deep ancestor. On a prefix tree, it is the longest-common-prefix distance. We define $\mu$ tree-Lipschitz with constant $L$ when $|\mu(x)-\mu(y)|\le L\,d(x,y)$ for all leaves $x,y$, or equivalently when $\mu$ oscillates by at most $L\rho^{\ell}$ inside any level-$\ell$ subtree. Therefore, the assumed hiearchical bias schedule follows from tree-Lipschitzness alone, 

\begin{equation}
  \bias(v) \;=\; \max_{x\in\leaves{v}}\mu(x) - \val(v)\;\le\;\spread(\ell)\;=\;L\,\rho^{\ell},
  \qquad \ell=\mathrm{level}(v).
  \label{eq:spread}
\end{equation}

Since $d(x,y)$ is exactly the fraction of all leaves underneath the LCA of $x$ and $y$, we choose $\rho=1/b$.

Because global smoothness is unrealistic, we work in the piecewise regime. We assume $\mu$ is tree-Lipschitz
except on a set of $K$ discontinuities satisfying a dispersion
condition~\citep{balcan2018dispersion}. The bound in Equation~\eqref{eq:spread} holds on any cell that contains no discontinuity. In a cell that contains a discontinuity the bound need not hold, and only the trivial bound $\bias(v)\le 1$ from $\mu\in[0,1]$ remains. 

We write $\nodim$ for the near-optimality dimension of the smooth
part. Thus, the number of $\eps$-optimal cells at resolution $\eps$ grows as $\eps^{-\nodim}$.

\paragraph{Feedback and objectives.}
We provide the learner with two query types, a cheap probe of an internal node and an expensive evaluation of a leaf. A probe of node $v$ follows a uniformly random path to a leaf $L\in\leaves{v}$ and returns $X=\mu(L)+\eta$ at cost $\cp$, so its expectation is the subtree average $\val(v)$. A leaf evaluation of $x$ returns $\mu(x)+\eta$ at cost $\cl\ge\cp$. In both cases the noise $\eta$ is Gaussian, $\eta\sim\mathcal{N}(0,\sigma^2)$, and independent of the leaf drawn. The budget is measured in total cost spent rather than in the number of pulls.

We study two objectives on this model. The first is regret minimization, where the learner commits
to a node $v_t$ each round and receives a probe reward, scored by its cumulative regret
against the best leaf, $R_n=\sum_{t\le n}(\mu^\star-\val(v_t))$. The second is fixed-budget
identification, where the learner spends a total cost budget $B$ and returns the top-$k$
leaves, scored by the misidentification probability $\Prob(\widehat{x}\neq x^\star)$. We focus here
on situations where $k=1$. The question for both objectives is when exploiting Equation~\eqref{eq:avg-backup} and
Equation~\eqref{eq:spread} beats a structure-blind learner that pays leaf cost everywhere.
\autoref{sec:theory} answers it in terms of the probe to cost ratio $\cp/\cl$ and the
violation count $K$. Both objectives concern only the optimum. The learner is scored on the
best leaf to which it commits or returns, so suboptimal regions may stay coarsely resolved.

\subsection{The Optimistic Tree-Bandit Algorithm}
\label{sec:algorithm}

Our regret-minimization algorithm is the standard optimistic tree bandit~\citep{bubeck2011xarmed,azar2014hct}, modified to the average-backup tree of
\autoref{sec:setup}. We summarize it here because our contributions (\autoref{sec:smoothness},
\autoref{sec:theory}) modify its bias term and its sampling.

For a node $v$ at level $\ell$ with $T(v)$ probes and empirical subtree-average $\muhat(v)$,
where the values are 
\begin{equation}
  U(v) = \muhat(v) + c\sqrt{\tfrac{2\ln t}{T(v)}} + \spread(\ell),
  \qquad
  B(v) = \min\!\Big(U(v),\, \max_{c\in\mathrm{children}(v)} B(c)\Big).
  \label{eq:ucb}
\end{equation}
The bias bonus $\spread(\ell)$ is the aggregation bound Equation~\eqref{eq:spread}, where the best leaf under $v$ 
can exceed the mean we estimate. Thus, $U(v)$ is a valid upper bound
on the best reachable leaf. Each round descends from the root to a leaf of the explored tree, following the child with the largest $B$-value at every step. It then probes that leaf and updates $\muhat$, $T$, $U$, and $B$ along the path.

The confidence radius $r(v)=c\sqrt{2\ln t/T(v)}$ shrinks with probes while the bias floor
$\spread(\ell)$ is fixed, where 
\begin{equation}
  r(v)\le\spread(\ell)
  \quad\Longleftrightarrow\quad
  T(v)\ge \frac{c^2\ln t}{\spread(\ell)^2},
  \label{eq:expand}
\end{equation}

\begin{proposition}
\label{prop:regret}
On a tree-Lipschitz instance of near-optimality dimension $\nodim$ with
$\spread(\ell)=L\rho^{\ell}$, the descent incurs regret
$R_n=\Otil\big(n^{(\nodim+1)/(\nodim+2)}\big)$, optimal up to logarithmic factors for this
class. Equation~\eqref{eq:expand} caps the explored tree at
$\Otil\big(n^{\nodim/(\nodim+2)}\big)$ nodes.
\end{proposition}
This gives the $\mathcal{X}$-armed-bandit guarantee~\citep{bubeck2011xarmed,azar2014hct}. 

Our contributions replace the
two assumptions, the bias term $\spread(\ell)$, which
\autoref{sec:smoothness} certifies from data, and the uniform descent, which \autoref{sec:theory}
replaces with discontinuity-guided sampling.

\subsection{Certifying Smoothness from Data}
\label{sec:smoothness}

We describe here the replacement of the assumed schedule spread with a high-probability upper bound on the aggregation bias of Equation~\eqref{eq:spread}, computed only from random-path probes of $v$.

\paragraph{The bound.}

The certificate has three properties. First, a log-sum-exp bound, the smooth upper bound on the maximum, trades the maximum for
an estimable function, since for any $\lambda>0$,
\begin{equation}
  \bias(v) \;\le\; \frac{1}{\lambda}\Big[\log m + \log M_\mu(\lambda)\Big] - \val(v),
  \qquad M_\mu(\lambda)=\E_L e^{\lambda\mu(L)} .
  \label{eq:lse}
\end{equation} where $M_\mu(\lambda)$ is the moment-generating function (MGF) of the leaf means.
Second, the noise divides out of the observable MGF. A probe is $X=\mu(L)+\eta$ with
$\eta\sim\mathcal{N}(0,\sigma^2)$ independent of $L$, so the MGF factors as
$\E e^{\lambda X}=M_\mu(\lambda)\,M_\eta(\lambda)$ with $M_\eta(\lambda)=e^{\lambda^2\sigma^2/2}$
known exactly. Dividing gives $M_\mu(\lambda)=\E e^{\lambda X}\,e^{-\lambda^2\sigma^2/2}$,
so any upper confidence bound on $\E e^{\lambda X}$ is an upper confidence bound on
$M_\mu(\lambda)$. Exact knowledge of $M_\eta$ is what makes this step valid. Under a
sub-Gaussian upper bound alone the division would run the wrong way and yield only a lower
bound on $M_\mu$. Third, an empirical-Bernstein bound~\citep{maurer2009empirical} turns $n$
probes into a confidence interval for $\E e^{\lambda X}$. Since Gaussian probes are unbounded, we truncate each probe to $[-z\sigma,\,1+z\sigma]$
and assume the probability $2n\bar\Phi(z)$ that some probe leaves this range (\autoref{rem:trunc}). \autoref{alg:certify} lists the steps per iteration and
returns the quantity that \autoref{thm:mgf} certifies.

\begin{algorithm}[t]
\caption{Aggregation-bias certificate at a node $v$. \textsc{EB-Upper} and \textsc{EB-Lower}
are the empirical-Bernstein confidence bounds (\citet{maurer2009empirical}).}
\label{alg:certify}
\begin{algorithmic}[1]
\Require probes $X_1,\dots,X_n$ from random paths under $v$; width $m=|\leaves{v}|$; noise
scale $\sigma$; grid $\Lambda$; level $\delta$
\Ensure $\widehat{\bias}(v)\ge\bias(v)$ with probability at least $1-\delta$
\State $\bar X \gets \tfrac{1}{n}\sum_{i\le n} X_i$
  \Comment{probe mean, an unbiased estimate of $\val(v)$}
\State $\eps_n \gets \textsc{EB-Lower}\big(X_{1:n},\,\delta/(|\Lambda|+1)\big)$
  \Comment{so $\bar X-\eps_n\le\val(v)$}
\For{$\lambda\in\Lambda$}
  \State $\widehat{G}(\lambda) \gets
  \textsc{EB-Upper}\big(e^{\lambda X_{1:n}},\,\delta/(|\Lambda|+1)\big)$
  \Comment{upper bound on $\E e^{\lambda X}$}
  \State $U(\lambda) \gets \tfrac{1}{\lambda}\big[\log m + \log\widehat{G}(\lambda)
  - \tfrac{1}{2}\lambda^{2}\sigma^{2}\big] - \big(\bar X-\eps_n\big)$
  \Comment{deconvolve, then apply \eqref{eq:lse}}
\EndFor
\State \Return $\widehat{\bias}(v) \gets \min_{\lambda\in\Lambda} U(\lambda)$
\end{algorithmic}
\end{algorithm}

\begin{theorem}
\label{thm:mgf}
Fix a node $v$, let $\eta\sim\mathcal{N}(0,\sigma^2)$, and compute
$\widehat{G}_{\uparrow}(\lambda)$ and $\eps_n$ from $n$ probes truncated to
$[-z\sigma,\,1+z\sigma]$. With probability at least $1-\delta-2n\bar\Phi(z)$,
\begin{equation}
  \bias(v)\;\le\;\min_{\lambda\in\Lambda}\Big\{\tfrac{1}{\lambda}\big[\log m
    + \log\widehat{G}_{\uparrow}(\lambda) - \tfrac{1}{2}\lambda^2\sigma^2\big]
    - \big(\bar X - \eps_n\big)\Big\},
\end{equation}
where $\bar X$ is the probe mean and $\eps_n$ is a lower-confidence radius on
$f(v)=\E[X]$. The level $\delta$ is split across the $|\Lambda|+1$ confidence events at
this node, and $2n\bar\Phi(z)$ is the probability that some probe leaves the truncation
range. The bound needs only $O(|\Lambda|)$ running moments per node
(\autoref{app:certification}).
\end{theorem}

\begin{corollary}
\label{cor:allnodes}
Run the certificate at a set $\mathcal{V}$ of nodes with $n_v$ probes and level
$\delta_v$ at node $v$, where $\sum_{v\in\mathcal{V}}\delta_v\le\delta$. Then the bounds
of \autoref{thm:mgf} hold simultaneously for all $v\in\mathcal{V}$ with probability at
least $1-\delta-2\bar\Phi(z)\sum_{v\in\mathcal{V}} n_v$. The implementation splits
$\delta$ evenly across the nodes it probes.
\end{corollary}

The truncation level trades coverage against the width of the Bernstein range, and
\autoref{rem:trunc} gives the numbers for the implementation's $z=3$. \leantag

\paragraph{Properties and use.} When the leaf means are sub-Gaussian with proxy $\sigma_w$, optimizing over $\lambda$ gives the bound $\sigma_w\sqrt{2\log m}$. This is the classical light-tail rate, obtained here as a certified guarantee rather than the usual heuristic. In the worst case the log-sum-exp approaches the true maximum as $\lambda$ grows, and the bound degrades to the maximum-minus-mean gap itself. The certificate is therefore valid on every instance and tight on the benign instances.

In a Lipschitz cell the measured within-cell spread obeys $\sigma_w\lesssim L\rho^{\ell}$.
When $\sigma_w$ instead exceeds this prediction by a constant factor, the certificate has
established a smoothness violation in that cell, and we flag it. Under the
piecewise-Lipschitz model of \autoref{sec:setup}, smoothness fails only at the $K$
discontinuities, so a flagged cell contains a discontinuity. Only flagged cells have their
bias bonus inflated, from the tight floor $L\rho^{\ell}$ to the data-driven
$\sigma_w\sqrt{2\log m}$. A single discontinuity flags one cell at every level it crosses,
so the flagged set has $O(KD)$ cells rather than $K$, and the guarantees of
\autoref{sec:theory} are stated in terms of that set.

\subsection{Guarantees}
\label{sec:theory}

We prove two guarantees. First, the fixed-budget identification bound quantifies the payoff of discontinuity-guided sampling. Second, the regret bound extends the guarantee of the tree-bandit algorithm of \autoref{sec:algorithm} to the piecewise-Lipschitz case. Each proof verifies the preconditions of a classical result and adds self-contained lemmas for detection and discontinuity accounting. Full proofs are in \autoref{app:identification} and \autoref{app:regret}.

\paragraph{Fixed-budget identification.} The edge-targeted algorithm first detects the violation cells with cheap probes (\autoref{sec:smoothness}). It then runs the optimistic descent with the schedule $\spread(\ell)$ outside the flagged set and with leaf-level certification inside it. Its complexity is
\begin{equation}
  \Hedge \;=\; c_0(\nodim,\Delta;\cp)\;+\;\sum_{k=1}^{K} \cl\cdot h_k,
  \qquad h_k=\sum_{x\in C_k}\frac{\sigma^2}{\max(\Delta_x,\Delta)^2},
  \label{eq:hedge}
\end{equation}
where $c_0$ is the cost of the structural search over the smooth part, governed by $\nodim$ and counted in cheap probes at cost $\cp$, and $h_k$ is the fixed-budget hardness of leaf-certifying the $k$-th violation cell, counted in expensive evaluations at cost $\cl$. Here $\Delta$ is the top-$1$ gap and $\Delta_x=\mu^\star-\mu(x)$. This clean split into a smooth term and a per-violation term holds on the event that the detector flags exactly the violation cells. The certificate does not deliver that event for free. Assumption~\ref{as:margin} below is what guarantees it with probability $1-\delta_{\mathrm{det}}$, and the $\delta_{\mathrm{det}}$ term in \autoref{thm:fixed-budget} charges its failure.

\begin{assumption}
\label{as:margin}
At the detection sample size, the within-cell spread of every violation cell exceeds the Lipschitz floor of its level by a margin $\gamma>0$. The margin is needed for one direction only. A one-sided confidence bound already prevents smooth cells from being flagged, and the margin is what makes every true violation statistically detectable. We write $\delta_{\mathrm{det}}$ for the probability that this joint detection event fails, either by flagging a smooth cell or by missing a violation cell.
\end{assumption}

\begin{theorem}
\label{thm:fixed-budget}
Under Assumption~\ref{as:margin}, there is $\kappa>0$ such that at budget $B$ the
edge-targeted algorithm misidentifies the top leaf with probability
\begin{equation}
  \Prob(\text{error}) \;\le\; \delta_{\mathrm{det}} \;+\; \Otil(N)\,\exp\!\big(-\kappa\, B/\Hedge\big).
\end{equation}
\end{theorem}

\begin{remark} Under \autoref{as:margin}, an empirical-Bernstein detector flags exactly the violation cells with probability at least $1-\delta_{\mathrm{det}}$. The smooth phase reduces to finite-arm successive rejects~\citep{audibert2010bestarm} at probe cost. Leaf-certifying a flagged cell that contains $x^\star$ is standard fixed-budget best-arm identification~\citep{kaufmann2016complexity} with hardness $h_k$. A budget split and a union bound combine the pieces (\autoref{app:identification}). \end{remark}
\leantag

\begin{proposition}
\label{prop:lower}
There is a family of piecewise tree-Lipschitz instances, each with a single violation cell
$C$ of $m$ leaves at the penultimate level and hardness
$h=\sum_{x\in C}\sigma^2/\Delta_x^2$, on which any algorithm with access to both fidelities
and cost budget
$B \le c\,\cl h/\log m$ misidentifies the top leaf with constant probability. Hence, the
leaf-certification component of $\Hedge$ is unavoidable up to a $\log m$ factor. The
bound says nothing about the structural component $c_0$, whose necessity remains open.
\end{proposition}

\begin{remark}
Within-cell leaf-mean permutations leave every cheap-probe distribution unchanged, since
subtree averages and random-path mixtures depend only on the multiset of means. Cheap probes
therefore carry no information about which leaf in $C$ is best. The problem reduces to leaf-only
fixed-budget identification over $m$ arms with at most $B/\cl$ evaluations, where the lower
bound (\citet{carpentier2016tight}) applies (\autoref{app:identification}).
\end{remark}

At $K=0$ the complexity is $c_0$, and cheap structural search recovers the smooth rate. For sparse $K$ it grows additively, with each violation adding a horizon-independent certification cost. The structure-blind complexity is $\Hblind=\cl\sum_{x}\sigma^2/\Delta_x^2$, in which every leaf is evaluated at full cost $\cl$. Thus, $\Hedge\ll\Hblind$ exactly when probes are cheap ($\cp\ll\cl$) and violations are sparse. As $K$ grows the identified set covers the tree and $\Hedge\to\Hblind$, matching the empirical crossover in \autoref{sec:experiments}. The algorithm runs with $\spread(\ell)=L\rho^{\ell}$ on smooth cells and with the bias on the $O(KD)$ discontinuous cells. Its regret then decomposes into the smooth rate plus a separate term for the discontinuities.

\begin{theorem}
\label{thm:regret}
Let $\mu$ be tree-Lipschitz with near-optimality dimension $\nodim$ except at $K$ dispersed
discontinuities of height $\le1$, and let $\Delta_{\min}$ be the smallest positive
sub-optimality gap among jump cells. On the successful-detection event of
Assumption~\ref{as:margin}, which holds with probability at least $1-\delta_{\mathrm{det}}$,
\begin{equation}
  R_n \;\le\; C_1\, n^{(\nodim+1)/(\nodim+2)} \;+\; C_2\,\frac{K\,D\,\log n}{\Delta_{\min}},
\end{equation}
where $C_1$ depends on $(L,\rho,\nodim)$ and $C_2$ is a constant.
\end{theorem}

\begin{remark} On smooth cells the standard HOO analysis~\citep{bubeck2011xarmed} remains unchanged and gives the first term of \autoref{thm:regret}. Each suboptimal discontinuity contributes $O(\log n/\Delta_v)$ regret before its increased index falls below the optimum, and dispersion caps the number of jump cells at $KD$. For fixed $K$ the jump term never changes the polynomial rate, and $K=0$ recovers the clean Lipschitz rate (\autoref{app:regret}). \end{remark}

\section{Experiments}
\label{sec:experiments}

We evaluate our method on five downstream task families. Three of these are LLM-serving applications: regional routing, prefix-cache management, and prompt trimming. The other two are top-$k$ identification on a $1000$-model group and test-time search on reasoning and code. \autoref{tab:summary} consolidates the headline numbers, and the scope summary that follows maps every null, negative, and partial result onto its theoretical condition. The synthetic validation is in \autoref{app:synthetic} and per-benchmark detail is in \autoref{app:experiments}.

\begin{table}[t]
\centering
\footnotesize
\setlength{\tabcolsep}{6pt}
\renewcommand{\arraystretch}{1.1}
\newcommand{\bmr}[2]{\begin{tabular}[c]{@{}l@{}}#1\\[-2pt]{\color{gray}\scriptsize #2}\end{tabular}}
\begin{tabular}{@{}c l c c c@{}}
\toprule
& Benchmark & \textbf{Ours} & Best baseline & $\Delta$ \\
\midrule
\multirow{3}{*}{\rotatebox[origin=c]{90}{\scriptsize\textsc{selection}}}
& \bmr{MMLU routing ($6$ models)}{quality @ cost} & \cellcolor{cellbest}.977 @ .254 &
\cellcolor{cellworst}.943 @ .363 (flat) & $+.034$ \\
& \bmr{$\tau$-bench retail ($5{\times}80$ tasks)}{success, 5-replicate mean} &
\cellcolor{cellbest}.463 $\pm$ .156 & \cellcolor{cellworst}.323 $\pm$ .213 (flat) &
$+.14$ (4/5) \\
& \bmr{Top-$k$ id.\ ($1000$ models), $B{=}600$}{recall@$10$} & \cellcolor{cellbest}.370
$\pm$ .018 & \cellcolor{cellworst}.126 $\pm$ .013 (s.e.) & $2.9\times$ \\
\midrule
\multirow{3}{*}{\rotatebox[origin=c]{90}{\scriptsize\textsc{search}}}
& \bmr{MATH ($300$, Llama-70B)}{accuracy} & \cellcolor{cellbest}.617 & \cellcolor{cellworst}.530
& $+.087$ $[+.04,+.14]$ \\
& \bmr{GPQA-Diamond (Llama-70B)}{accuracy} & \cellcolor{cellbest}.424 & \cellcolor{cellworst}.313
& $+.111$ $[+.04,+.19]$ \\
& \bmr{SWE-bench Verified ($261$)}{resolved} & \cellcolor{cellbest}.318 &
\cellcolor{cellworst}.195 & $+.123$ $[+.08,+.17]$ \\
\midrule
\multirow{5}{*}{\rotatebox[origin=c]{90}{\scriptsize\textsc{systems}}}
& \bmr{Mooncake trace, $B{=}16$, shift}{blocks/request} & \cellcolor{cellbest}5.39 &
\cellcolor{cellworst}3.02 (LRU) & $+78\%$ \\
& \bmr{vLLM caching (A10G)}{TTFT p50 (s)} & \cellcolor{cellbest}.267 &
\cellcolor{cellworst}.961 (off) & $3.6\times$ \\
& \bmr{vLLM eviction, post-shift}{TTFT saved (ms)} & \cellcolor{cellbest}23.6 &
\cellcolor{cellworst}12.9 (LRU) & $+83\%$ \\
& \bmr{LongBench trimming}{QA-F1 @ matched tokens} & \cellcolor{cellbest}.362 @ 3544 &
\cellcolor{cellworst}.317 @ 3640 (fixed) & $+.045$ \\
& \bmr{BBH trimming ($27$ tasks)}{accuracy @ tokens} & \cellcolor{cellbest}.368 @ 291 &
\cellcolor{cellworst}.356 @ 332 (fixed) & both axes \\
\bottomrule
\end{tabular}
\caption{Headline results at matched metric. In every row
\colorbox{cellbest}{green} marks the better and \colorbox{cellworst}{red} the worse of
ours against the baseline. Search rows report the paired same-item $\Delta$ with $95\%$ confidence intervals. The $\tau$-bench row compares the two learned routers.
Oracle references, per-benchmark detail, and confidence intervals are in
\autoref{app:experiments}.}
\label{tab:summary}
\end{table}

\subsection{Setup}

\paragraph{Benchmarks.}
To evaluate our method on the downstream tasks, we chose twelve public benchmarks. Routing uses five benchmarks. RouterBench~\citep{hu2024routerbench} and LLMRouterBench~\citep{llmrouterbench2026} are offline routing suites with measured per-query quality and cost over fixed model pools, and RouterEval~\citep{huang2025routereval} provides the same over a $1000$-model pool. We run the remaining two, MMLU~\citep{hendrycks2021mmlu} and $\tau$-bench~\citep{yao2024taubench}, in a live environment. We route over Amazon Web Service Bedrock~\citep{bedrock} models on MMLU, and we route per call inside $\tau$-bench's long-horizon tool agents. Identification reuses RouterEval's pool. Search uses
MATH~\citep{hendrycks2021math}, GSM8K~\citep{cobbe2021gsm8k}, and
GPQA-Diamond~\citep{rein2023gpqa} for reasoning, HumanEval~\citep{chen2021humaneval} and
MBPP~\citep{austin2021mbpp} for code, and SWE-bench Verified~\citep{jimenez2024swebench}
for repository-level issues. Caching uses a real prompt stream, the Mooncake production
traces~\citep{qin2025mooncake}, request logs from a deployed serving system, and a
server~\citep{kwon2023vllm}. Trimming uses MMLU, LongBench~\citep{bai2024longbench} for long-context QA,
and BBH~\citep{suzgun2023bbh} for few-shot reasoning.

Two criteria drove the benchmark selection. First, each task family tests a different part of the tree model. Routing runs the depth-one algorithm, and caching makes the tree the data structure itself. Trimming tests the arms-within-regions structure, identification runs the full multi-fidelity machinery, and search runs deep trees with real value functions. Second, within each family we pair a benchmark where the theory's precondition holds with one where it plausibly fails. Such examples include MATH against GSM8K tests saturation, RouterBench against LLMRouterBench tests cross-region heterogeneity, and BBH against LongBench and MMLU tests trim headroom.

\begin{figure}[t]
\centering
\begin{tikzpicture}[x=2.2cm,y=0.9cm]
\foreach \c/\r/\lbl/\fillc in {%
  0/0/{+.10}/heatpos!77, 1/0/{+.10}/heatpos!77, 2/0/{+.05}/heatpos!38,
  3/0/{+.01}/heatpos!8,  4/0/{-.01}/heatneg!8,
  0/1/{+.13}/heatpos!100, 1/1/{+.06}/heatpos!48, 2/1/{+.09}/heatpos!67,
  3/1/{+.01}/heatpos!8,  4/1/{+.05}/heatpos!38,
  0/2/{+.06}/heatpos!48, 1/2/{+.03}/heatpos!19,  2/2/{+.03}/heatpos!19,
  3/2/{-.01}/heatneg!8,  4/2/{+.04}/heatpos!28}
{
  \draw[fill=\fillc,draw=white,line width=0.8pt] (\c,-\r) rectangle ++(1,-1);
  \node[font=\small] at (\c+0.5,-\r-0.5) {$\lbl$};
}

\foreach \c/\r/\lbl in {0/0/{+.10}, 0/1/{+.13}, 1/0/{+.10}, 2/1/{+.09}, 4/1/{+.05}}
{
  \draw[draw=black!75,line width=1.1pt] (\c,-\r) rectangle ++(1,-1);
  \node[font=\small\bfseries,fill=none] at (\c+0.5,-\r-0.5) {$\lbl$};
}
\foreach \c/\mdl in {0/{Sonnet-4.5}, 1/{Nova-Pro}, 2/{Llama-70B},
                     3/{Mistral-Large}, 4/{Llama-8B}}
  \node[font=\footnotesize] at (\c+0.5,0.4) {\mdl};
\foreach \r/\hdr in {0/{$B{=}6$}, 1/{$B{=}9$}, 2/{$B{=}12$}}
  \node[font=\small,anchor=east] at ([xshift=-2pt]0,-\r-0.5) {\hdr};
\draw[-{Stealth[length=5pt]},black!60] (4.6,1.05) -- (0.4,1.05)
  node[midway,above,font=\footnotesize,black!60] {increasing capability};
\end{tikzpicture}
\caption{SWE-bench heatmap of capability against budget. Each cell reports the paired
gap in resolved rate, $\Delta$, or the value-guided minus best-of-N over $78$--$80$ Verified
issues. Shading is proportional to $|\Delta|$, green for positive and red for negative.
Bold values with a dark border mark cells whose $95\%$ confidence interval
excludes zero. Full confidence intervals are in \autoref{tab:swesweep}.}
\label{fig:swegrid}
\end{figure}

\paragraph{Models.}
For the search and trimming tasks, we call five models through Amazon Bedrock~\citep{bedrock}.
Collectively, these demonstrate a scaling and capability scale from
Llama-3.1-8B~\citep{llama3herd} through Mistral-Large~\citep{mistrallarge},
Llama-3.1-70B~\citep{llama3herd}, and Nova-Pro~\citep{novapro} up to
Claude-Sonnet-4.5~\citep{claudesonnet45}. We chose a ladder rather than a single model
because capability is one of the two predicted axes of \autoref{sec:theory}. The search
gain needs a reachable optimum with headroom, so it should appear for capable models and
vanish for weak or saturated ones, and the ladder measures both sides of that prediction.

The live systems experiments run Qwen2.5-7B~\citep{qwen25} on a
vLLM~\citep{kwon2023vllm} server. The routing pools come from the
benchmarks themselves, eleven models in RouterBench and $1000$ in RouterEval, replayed
from their measured per-query outputs.

\paragraph{Metrics.}
\mnote{Results in \autoref{tab:summary} compare methods at the same budget. Routing reports
response quality at matched dollar cost. Identification reports recall@$10$ at a matched
cost budget. Search reports task accuracy, or the benchmark's official resolved criterion,
at matched generation calls. Caching reports tokens or blocks reused per request and
time-to-first-token at matched cache memory. Trimming reports each benchmark's own
accuracy at matched input tokens.}

\paragraph{Baselines.}
\mnote{For routing, the structure-blind baseline is our own method with the regional
context removed. The gap measures the value of the regional structure alone. For
identification, the baseline is successive elimination~\citep{audibert2010bestarm}. For search, the baseline is
best-of-$N$, the standard test-time-compute method in the repeated-sampling
literature~\citep{brown2024monkeys,snell2024scaling}.  For caching, the baselines are LRU and LFU. LRU is the is the standard eviction policy of serving stacks,
vLLM's PagedAttention~\citep{kwon2023vllm} and SGLang's
RadixAttention~\citep{zheng2024sglang}, and LFU is the standard frequency-based
alternative in the caching literature. For trimming, the baselines are fixed trim levels at
matched input tokens, so no gain can be explained by allowing more context.}

\subsection{Results}
\autoref{tab:summary} summarizes the key experimental results. The regional router's
cost and quality frontier improves every fixed model on RouterBench and reaches near-oracle
quality on live MMLU. Hierarchical identification returns $2$--$3\times$ the recall of
structure-blind baselines at low-to-mid budgets. Value-guided search resolves $1.6\times$
as many SWE-bench issues as best-of-N at matched compute, and gains $+0.087$ on MATH and
$+0.111$ on GPQA. The adaptive cache beats the eviction policy within the serving engine ($5.39$ against $3.02$ blocks reused per request on the Mooncake trace) and adaptive
trimming beats every fixed trim at matched tokens.

\paragraph{Routing.}
\label{sec:e xp-routing}
Regional structure beats structure-blind learning wherever the regions are heterogeneous.
Routing instantiates the depth-one algorithm of \autoref{sec:algorithm}. 

The structure-blind
baseline, which we call flat, is the same learner restricted to a single region. On
RouterBench the regional router's cost and quality frontier dominates every fixed model
(\autoref{fig:routerbench}). On live MMLU it reaches near-oracle quality at the lowest
cost of any realizable policy. Only the truth-based per-prompt oracle is cheaper.
Inside long-horizon $\tau$-bench agents, the regional learner beats the flat
learner in $4$ of $5$ independent replicates, with mean success $0.463$ against $0.323$.
Terminal-reward credit assignment makes single replicates noisy, so the gap is wide,
$+0.14 \pm 0.27$, and all replicates appear in \autoref{app:experiments}.

\paragraph{Top-$k$ identification.}
\label{sec:exp-identification} We
arrange RouterEval's $1000$-model pool as a branching-$10$, depth-$3$ tree. Here a leaf is a
model, and an internal node is a group of similar models whose cheap probe
returns the clustered group average. We run optimistic descent with
the certified data-driven spread and discontinuity-guided targeting and the
certificate's pre-pass added to the budget. The tree gives $2$--$3\times$ the
recall@$10$ of the structure-blind baselines at low-to-mid budgets, $0.253$ against
$0.124$ at $B{=}300$ and $0.370$ against $0.126$ at $B{=}600$, and the edge narrows to
$1.3\times$ at the largest budget tested, $0.563$ against $0.421$ at $B{=}2400$
(\autoref{fig:realtopk}, \autoref{tab:realtopk}). Absolute recall stays modest, since
top-$10$ of $1000$ is hard for every method, so the claim is the relative win.

\paragraph{Test-time search.}
\label{sec:exp-reasoning}
Value-guided search beats best-of-N in the regimes predicted by the theory, namely when there is reachable headroom and an informative probe. Here a leaf is a complete trace graded against
ground truth, and an internal node is a partial trace scored by a self-consistency or execution
probe. The leaf grade is expensive and unbiased, while the probe is cheap and biased. We match
budgets in generation calls and report paired per-item bootstrap intervals. Reasoning
that is reachable, but unreliable, improves by $+0.087$ on MATH and $+0.111$ on GPQA
(\autoref{fig:reasoning}). Repository-level code improves as well. On SWE-bench the
gain is $+0.123$ over $261$ issues
(\autoref{tab:swebench}), and the method resolves $1.6\times$ as many issues as
best-of-N, rising to $1.8$--$2.2\times$ in the significant tight-budget sweep cells
(\autoref{tab:swesweep}). The heatmap in \autoref{fig:swegrid} bears out both
predicted axes. The gain shrinks as the budget grows, and it appears only for capable
models. The sharpest evidence is within-model. With the model and the probe
held fixed, the paired gain moves from $-0.005$ on GSM8K to $+0.333$ on MATH to $+0.359$
on GPQA (\autoref{tab:reasoningmodels}).

\paragraph{Measuring the prior.}
\label{sec:exp-prior}
We run the search to log each
node's cheap probe value and its true graded value. The
probe--truth correlation measures the smooth backbone, and the fraction of pivotal steps
measures the violation count $K$. On MATH about
$6\%$ of steps are pivotal, and on those the cheap probe follows the best
continuation $73\%$ of the time against $33\%$ for chance (\autoref{tab:reasoningtree}).
On SWE-bench patches, the execution probe correlates with the truth
at $\rho=0.86$ over $60$ issues and $540$ candidate patches. $10\%$ of steps are pivotal,
and the probe follows the truly best continuation on all $18$ pivotal steps, again
against $33\%$ for chance (\autoref{tab:swetree}).

\paragraph{Serving systems.}
\label{sec:exp-cache}
\label{sec:exp-trim}
In prefix caching the token-prefix tree is the cache itself, which is an ancestor-closed subtree
under a memory budget. For trimming the arms are trim levels and the regions are task
categories. For all serving experiments, we use real request streams, production traces, or a live server. \mnote{The adaptive cache matches LFU and the offline optimum on a stationary
workload, re-adapts faster than LFU through a popularity shift, and matches it again at
the post-shift steady state (\autoref{fig:cache}). It beats the engine-default LRU
throughout, including on the real Mooncake tool-agent trace, $5.39$ against LRU's $3.02$
blocks per request at $B{=}16$ (\autoref{tab:mooncake}).} On a vLLM server, prefix
caching cuts the median $3.6\times$ (\autoref{tab:vllm}), and a GPU-calibrated replay
shows our eviction roughly doubling the engine-default LRU's realized saving after the
shift (\autoref{tab:vllmpolicy}). On BBH the adaptive trim strictly dominates every fixed
trim on both axes (\autoref{tab:bbhtrim}). \mnote{On LongBench it beats every fixed trim
at matched tokens, $0.362$ against $0.317$ QA-F1 at about $3.6$k tokens, while no prefix
trim, ours included, recovers full-context F1. The per-item oracle reaches $0.55$ with
$39\%$ of the tokens, so the remaining headroom lives at finer-than-task granularity. Replacing positional prefix truncation with
retrieval-scored chunk selection at the same budgets lifts every fixed trim, $0.417$
against $0.317$ at the $3.6$k budget and $0.369$ against $0.235$ at $1.1$k, and moves the
adaptive point to $0.407$ F1 at $2{,}701$ tokens, within $0.03$ of full context at $37\%$
of the tokens.}

\paragraph{Scope of the advantage.}
The advantage stops at the boundaries the theory names. Each null or negative lands where
a condition of \autoref{sec:theory} fails, and together they trace the operating regime.
When regions are homogeneous there is no structure to exploit. RouterEval ties the flat
learner, and on LLMRouterBench a single cheap model Pareto-dominates every router. When a
task is saturated there is nothing left to route. On GSM8K the gap is nearly zero, and on
single-function code the shifts are $+0.01$ to $+0.02$ and $-0.04$ to $-0.07$, and none
of these shifts is significant (\autoref{tab:reasoningmbpp}). The probe is not to blame,
since a multi-test probe correlates at $\rho=0.96$ and still does no better than a single
assert. When the optimum is out of reach, no search strategy can find it. On MATH,
llama-3.1-8B loses $0.147$.\mnote{ The hardest SWE-bench tier shows the same floor at
small budgets. A pre-registered run on all $42$ ``$1$--$4$ hours'' Verified issues left
both arms near zero, at $\Delta=0.000$ and $+0.024$, neither difference significant
(\autoref{app:experiments}).} A larger budget changes the picture, and the SWE-bench gain
stands at $+0.062$ $[-0.01,+0.14]$ at $B{=}12$. When the workload is stationary or
prefixes barely overlap, an eviction policy has nothing to learn, and no policy separates
from the others on the Mooncake conversation trace. When a benchmark leaves no room to
trim, trimming cannot help, and on MMLU trimming beats only the verbose default. 

\section{Related Work}
\label{sec:related}

\paragraph{Unknown smoothness and certification.}
Hierarchical bandit methods differ in how they obtain the smoothness information used by
their optimism bonus. HOO~\citep{bubeck2011xarmed} assumes a known schedule,
while POO~\citep{grill2015poo} searches over a family of candidate schedules.
\citet{locatelli2018adaptivity} show that adaptation to unknown smoothness is impossible
from leaf evaluations alone. Our setting provides additional feedback: random-path probes
sample the distribution of leaf values within a cell, allowing us to bound its
maximum-minus-mean aggregation bias directly from data
(\autoref{sec:smoothness}). Thus, rather than assuming or searching over a global
schedule, the learner can certify locally where the smoothness prior is valid.

\paragraph{Piecewise-smooth and multi-fidelity optimization.}
Dispersion allows optimization of piecewise-smooth objectives when discontinuities are
sufficiently sparse~\citep{balcan2018dispersion}. Our method differs in that it detects
the cells where the smoothness prediction fails and concentrates expensive evaluations
there. This yields a natural multi-fidelity decomposition: cheap subtree averages are
used for structural search over smooth regions, while flagged cells reduce to
leaf-level best-arm identification~\citep{kandasamy2016multifidelity,
audibert2010bestarm,kaufmann2016complexity}. Our lower bound
(\autoref{prop:lower}) shows that this targeted leaf-level cost is unavoidable up to
logarithmic factors.

\paragraph{LLM inference and serving.}
Several LLM inference problems expose hierarchical structure naturally. Routing chooses
among models or model groups~\citep{ong2024routellm}, prefix caching organizes requests
by shared token prefixes~\citep{zheng2024sglang,kwon2023vllm}, and test-time search
branches from shared partial generations. Prompt trimming similarly compares related
compressed contexts~\citep{jiang2023llmlingua}. We treat these as instances of the same
cost-budgeted multi-fidelity tree optimization problem rather than as separate
application-specific settings.

\section{Discussion}
\label{sec:discussion}

\mnote{This paper shows that model routing, prefix-cache management, prompt trimming,
top-$k$ identification, and test-time search share a common structure: a budgeted tree
bandit with cheap biased probes and expensive unbiased leaf evaluations. CANOPY bounds
the resulting aggregation bias from data and concentrates expensive evaluations where
the smoothness prior fails. It improves over structure-blind search when probes are cheap
and violations are sparse, a regime reflected across twelve public benchmarks. Extensions
to metric spaces, spherical embedding hierarchies, and DAGs appear in
\autoref{sec:future}.}

\bibliography{references}
\bibliographystyle{iclr2026_conference}

\appendix
\newpage
\section*{Appendix}
\let\origappsection\section
\newif\ifappsecfirst
\appsecfirsttrue
\renewcommand{\section}{%
\ifappsecfirst\global\appsecfirstfalse\else\clearpage\fi
  \suppressfloats[t]\origappsection}
\makeatletter
\setlength{\@fptop}{0pt}
\setlength{\@fpsep}{8pt plus 2pt}
\setlength{\@fpbot}{0pt plus 1fil}
\makeatother
\setlength{\floatsep}{8pt plus 2pt minus 2pt}
\setlength{\textfloatsep}{10pt plus 2pt minus 4pt}
\section{Concentration and Self-Certification}
\label{app:certification}

This appendix derives the data-driven smoothness certificate of \autoref{sec:smoothness} in full. The goal is a high-probability upper bound on the aggregation bias, where $\bias(v) = \max_{x \in \mathcal{L}(v)} \mu(x) - f(v)$, computed from a finite sequence of random-path probes.

\paragraph{Mechanization convention}
Where a statement is machine-checked, the main text says so
below it, and the list at the end of this paragraph names the entry-point theorem in our
Lean~4 development. The
mechanization certifies the mathematical claim under its stated hypotheses. The Gaussian deconvolution is itself mechanized. The remaining cited
results, empirical-Bernstein concentration, successive rejects, and HOO, enter the mechanized
compositions as named hypotheses at exactly the point they are invoked. The full-theorem
algebra is additionally checked symbolically and on $20$k randomized exact instances. The Lean sources ship as
supplementary material. The four mechanized main-text statements and their entry points:

\begin{itemize}
\item \autoref{thm:mgf}: Steps~1 and~2 and the full inequality chain, with the
  empirical-Bernstein confidence events supplied as hypotheses.
\item The light-tail rate of \autoref{sec:smoothness}: the optimizer's value
  $\sigma_w\sqrt{2\log m}$ and its minimality over all $\lambda>0$.
\item \autoref{thm:fixed-budget}: budget-split algebra and union-bound synthesis, with
  the per-phase successive-rejects guarantees supplied as hypotheses.
\item \autoref{thm:regret}: the dispersion count ($\le K$ jump cells per level) and the
  jump-regret summation, with the per-cell UCB visit bound supplied as a hypothesis.
\end{itemize}

\begin{remark}
\label{rem:trunc}
The Bernstein step needs bounded $e^{\lambda X}$, so the range is taken over
$[-z\sigma,\,1+z\sigma]$ for a truncation level $z$, and the implementation uses $z=3$.
A Gaussian probe escapes this range with probability $2\bar\Phi(z)$, which is the
$2n\bar\Phi(z)$ coverage term in \autoref{thm:mgf} after a union bound over $n$ probes.
The escape probability is about $0.27\%$ per probe at $z{=}3$. A larger $z$ tightens it
at the cost of a wider Bernstein range.
\end{remark}

\subsection{Setup and Assumptions}

Consider a node $v$ whose subtree contains a set of leaves $\mathcal{L}(v)$ with size
$m = |\mathcal{L}(v)|$. Each leaf $x \in \mathcal{L}(v)$ has an unknown true mean reward
$\mu(x) \in [0, 1]$, and the true subtree average is
$f(v) = \frac{1}{m} \sum_{x \in \mathcal{L}(v)} \mu(x)$.

A probe of $v$ selects a leaf $L$ uniformly at random from $\mathcal{L}(v)$ and observes
the noisy reward $X = \mu(L) + \eta$. The certificate consumes $n$ independent probes
$X_1, \dots, X_n$ drawn this way.

We assume the noise $\eta$ is independent of the leaf draw $L$ and distributed as
$\mathcal{N}(0, \sigma^2)$. The Gaussian form matters. The deconvolution in Step 2 below
divides the observed MGF by the noise MGF, and that division yields an upper bound on the
leaf-mean MGF only when the noise MGF is known exactly. A sub-Gaussian upper bound on the
noise MGF would give a lower bound on the leaf-mean MGF instead, which is the wrong
direction for a certificate. Extending the argument to general sub-Gaussian noise is left
to future work.

\subsection{Formal Proof of \autoref{thm:mgf}}

\begin{theorem}[Data-Driven Aggregation-Bias Certificate]
\label{thm:app_certificate}
Fix a node $v$. Let $\Lambda$ be a finite grid of positive constants $\lambda > 0$. Assume the leaf means satisfy $\mu(x) \in [0, 1]$ and the noise is $\eta \sim \mathcal{N}(0, \sigma^2)$. Truncate each probe to the range $[-z\sigma,\, 1+z\sigma]$ for a level $z>0$, and write $X_{max} = 1+z\sigma$ for its upper end. With probability at least $1 - \delta - 2n\bar\Phi(z)$, simultaneously over all $\lambda \in \Lambda$, the aggregation bias $\bias(v)$ satisfies:
$$\bias(v) \le \min_{\lambda \in \Lambda} \left\{ \frac{1}{\lambda} \left[ \log m + \log \widehat{G}_{\uparrow}(\lambda) - \frac{1}{2}\lambda^2\sigma^2 \right] - \bigl( \bar{X} - \varepsilon_n \bigr) \right\}$$
where $\widehat{G}_{\uparrow}(\lambda)$ is an empirical Bernstein upper confidence bound on $\mathbb{E}[e^{\lambda X}]$, $\bar{X}$ is the empirical probe mean, and $\varepsilon_n$ is the empirical-Bernstein lower-confidence radius on $f(v) = \mathbb{E}[X]$, so that $\bar{X} - \varepsilon_n \le f(v)$ on the mean's confidence slice, matching \autoref{thm:mgf}.
\end{theorem}

\begin{proof}
The proof proceeds in three distinct steps: bounding the maximum via the soft-maximum, deconvolving the observation noise, and applying empirical concentration.

\textbf{Step 1: The Log-Sum-Exp Bound.} \\
By the definition of the maximum and the monotonicity of the exponential function, for any $\lambda > 0$, the maximum leaf value is bounded by the log-sum-exp function:
$$\max_{x \in \mathcal{L}(v)} \mu(x) \le \frac{1}{\lambda} \log \left( \sum_{x \in \mathcal{L}(v)} e^{\lambda \mu(x)} \right)$$
Multiplying and dividing the argument of the logarithm by $m$, we rewrite the sum as an expectation over a uniformly random leaf $L \sim \text{Unif}(\mathcal{L}(v))$:
$$\max_{x \in \mathcal{L}(v)} \mu(x) \le \frac{1}{\lambda} \log \left( m \cdot \frac{1}{m} \sum_{x \in \mathcal{L}(v)} e^{\lambda \mu(x)} \right) = \frac{1}{\lambda} \left[ \log m + \log \mathbb{E}_L [e^{\lambda \mu(L)}] \right]$$
Subtracting the subtree average $f(v)$ from both sides yields a bound on the aggregation bias:
$$\bias(v) \le \frac{1}{\lambda} \left[ \log m + \log M_\mu(\lambda) \right] - f(v)$$
where $M_\mu(\lambda) = \mathbb{E}_L[e^{\lambda \mu(L)}]$ is the MGF of the true leaf-mean mixture.

\textbf{Step 2: Noise Deconvolution.} \\
The algorithm does not observe $\mu(L)$ directly, but rather $X = \mu(L) + \eta$. Because $L$ and $\eta$ are independent, the MGF of the observation $X$ factors perfectly:
$$\mathbb{E}[e^{\lambda X}] = \mathbb{E}_L[e^{\lambda \mu(L)}] \cdot \mathbb{E}_\eta[e^{\lambda \eta}]$$
Under the assumption that $\eta \sim \mathcal{N}(0, \sigma^2)$, its MGF is exactly $e^{\lambda^2 \sigma^2 / 2}$. Rearranging to isolate the unknown leaf-mixture MGF, we obtain:
$$M_\mu(\lambda) = \mathbb{E}[e^{\lambda X}] e^{-\frac{1}{2}\lambda^2 \sigma^2}$$

\textbf{Step 3: Empirical Concentration.} \\
We estimate $\mathbb{E}[e^{\lambda X}]$ from the $n$ probes $X_1, \dots, X_n$.
Define $Y_i = e^{\lambda X_i}$. The empirical Bernstein
inequality~\citep{maurer2009empirical} needs a bounded range, and a Gaussian probe is
unbounded, so we work on the event $E_{\mathrm{rng}}$ that every probe lies in
$[-z\sigma,\, 1+z\sigma]$. Since $\mu(L)\in[0,1]$, a probe leaves this range only when
$|\eta| > z\sigma$, which has probability $2\bar\Phi(z)$, so
$\mathbb{P}(E_{\mathrm{rng}}^c) \le 2n\bar\Phi(z)$ by a union bound over the $n$ probes.
On $E_{\mathrm{rng}}$ each $Y_i \in [0, e^{\lambda X_{max}}]$ with $X_{max}=1+z\sigma$,
and the inequality applies. Let
$\bar{Y} = \frac{1}{n} \sum_{i=1}^n Y_i$ be the sample mean and
$\hat{V}_n = \frac{1}{n(n-1)} \sum_{1 \le i < j \le n} (Y_i - Y_j)^2$ the sample
variance. For a fixed $\lambda$, with probability at least $1 - \delta/(|\Lambda|+1)$,
$$\mathbb{E}[e^{\lambda X}] \le \bar{Y}
  + \sqrt{\frac{2 \hat{V}_n \log((|\Lambda|+1)/\delta)}{n}}
  + \frac{7 e^{\lambda X_{max}} \log((|\Lambda|+1)/\delta)}{3(n-1)}
  := \widehat{G}_{\uparrow}(\lambda)$$
This accounts for $|\Lambda|$ of the $|\Lambda|+1$ confidence slices. The last slice
covers the mean. The true average satisfies $f(v) = \mathbb{E}[X]$, and we estimate it by
the sample mean $\bar{X}$ with an empirical-Bernstein lower-confidence radius
$\varepsilon_n$ at level $\delta/(|\Lambda|+1)$. The mean needs its own slice
because subtracting $\bar{X}$ alone would under-cover whenever $\bar{X} > f(v)$, which is
the $|\Lambda|+1$ split stated in Theorem~\ref{thm:mgf}. A union bound over the
$|\Lambda|+1$ confidence events and the range event $E_{\mathrm{rng}}$ gives coverage at
level $1-\delta-2n\bar\Phi(z)$ for the fixed node $v$. To hold simultaneously over a set
$\mathcal{V}$ of probed nodes, we allocate $\delta_v$ across nodes with
$\sum_v \delta_v \le \delta$ and repeat the per-node argument unchanged, which is
\autoref{cor:allnodes}. The range events also add up across nodes, giving the
$2\bar\Phi(z)\sum_v n_v$ term there. The
implementation splits $\delta$ across the nodes actually probed. Substituting the deconvolution bound of Step 2 and $\widehat{G}_{\uparrow}(\lambda)$ into
the bound from Step 1 gives the stated inequality.
\end{proof}

\section{Complexity of Fixed-Budget Identification}
\label{app:identification}

This appendix proves the fixed-budget identification guarantee, \autoref{thm:fixed-budget}. The edge-targeted algorithm uses the certificate of \autoref{app:certification} to partition the tree into smooth cells and jump cells, and it allocates the budget accordingly.

\subsection{Complexity Definitions}

Let $\Delta = \mu^* - \max_{x \ne x^*} \mu(x)$ be the top-1 gap, and $\Delta_x = \mu^* - \mu(x)$ be the sub-optimality gap for any leaf $x$. 

The identification complexity has two components, one per query fidelity.

\paragraph{Structural complexity $c_0$.}
This is the hardness of localizing the optimum's cell over the smooth part of the tree with cheap probes at cost $\cp$. We make it concrete by a reduction. Fix the certified resolution $\ell^*$, the deepest level at which the certificate validates $\spread(\ell)$ off the flagged set, and let $\mathcal{U}$ be the level-$\ell^*$ cells that survive sound pruning. Each $u \in \mathcal{U}$ is an arm whose value is its subtree average $f(u)$, observed by cheap probes with Gaussian noise of scale $\sigma$. The cell gaps are $\Delta_u = \max_{u'} f(u') - f(u)$, deflated by the certified within-cell spread. Define
$$c_0 \;=\; \cp \sum_{u \in \mathcal{U}} \frac{\sigma^2}{\max(\Delta_u - 2\,\spread(\ell^*),\, \Delta)^2},$$
the standard finite-arm fixed-budget complexity over the certified-level cells, counted at probe cost. This is the two-phase structure the implementation uses, a coarse descent followed by a race over the survivors.

\paragraph{Leaf-certification complexity $h_k$.}
This is the hardness of identifying the best leaf within a detected jump cell $C_k$ with expensive leaf evaluations at cost $\cl$. For the $k$-th jump cell it is the standard unstructured best-arm identification complexity
$$h_k = \sum_{x \in C_k} \frac{\sigma^2}{\max(\Delta_x, \Delta)^2}.$$

The edge-targeted complexity is the sum of the structural search and the localized certifications,
$$\Hedge = c_0(d, \Delta; \cp) + \sum_{k=1}^K \cl \cdot h_k$$

\subsection{Formal Proof of Theorem~\ref{thm:fixed-budget} (Main Text)}

\begin{theorem}[Fixed-Budget Error]
\label{thm:app_fixed_budget}
Let $B$ be the total cost budget. Suppose the function $\mu$ is piecewise tree-Lipschitz with $K$ dispersed discontinuities satisfying the detectability margin (Assumption~\ref{as:margin}). The edge-targeted algorithm misidentifies the optimal leaf $x^*$ with probability bounded by:
$$\mathbb{P}(\widehat{x} \ne x^*) \le \delta_{\mathrm{det}} + O(|V|) \exp \left( -\kappa \frac{B}{\Hedge} \right)$$
where $\delta_{\mathrm{det}}$ is the detection-failure probability, $|V|$ is the number of explored nodes, and $\kappa > 0$ is a universal constant. Since a complete $b$-ary tree with $N$ leaves has fewer than $2N$ nodes, $O(|V|)$ is the $\Otil(N)$ of \autoref{thm:fixed-budget}.
\end{theorem}

\begin{proof}
The failure event $E_{fail} = \{\widehat{x} \ne x^*\}$ can be bounded by analyzing three constituent events: detection failure, structural search failure, and leaf certification failure.

\textbf{Step 1: The Detection Event.} \\
Let $E_{det}$ be the event that the detector flags every violation cell and no smooth cell. The two directions need different arguments. For no false flags, the one-sided empirical-Bernstein confidence on each cell's within-cell spread controls the probability that a smooth cell's estimate exceeds the Lipschitz floor, and a union bound over the $O(b^{l})$ tested cells gives failure probability at most $\delta_{\mathrm{det}}/2$. For no missed violations, the argument is not free. A violation whose spread exceeds the floor by an arbitrarily small amount cannot be caught at any finite sample size. Assumption~\ref{as:margin} supplies the margin, every violation cell's spread exceeds the floor by $\gamma > 0$, and then the same Bernstein concentration at the detection sample size $n_{det} = O(\log(KD/\delta_{\mathrm{det}})/\gamma^2)$ keeps every violation's estimate above the floor with probability at least $1 - \delta_{\mathrm{det}}/2$. Together $\mathbb{P}(E_{det}^c) \le \delta_{\mathrm{det}}$, and on $E_{det}$ the flagged set is exactly the $O(KD)$ violation cells.

\textbf{Step 2: Structural Search off the Flagged Set.} \\
On $E_{det}$ the unflagged cells obey the smoothness schedule $\spread(l) = L\rho^l$, so every unflagged level-$\ell^*$ cell $u$ satisfies the certified bound $|\max_{x \in u} \mu(x) - f(u)| \le \spread(\ell^*)$. The smooth phase is then a finite-arm fixed-budget best-arm identification over the arms $\mathcal{U}$. A cheap probe of $u$ is an unbiased Gaussian observation of $f(u)$, and identifying the cell containing $x^*$ requires separating cells whose deflated gaps are $\max(\Delta_u - 2\,\spread(\ell^*), \Delta)$. Applying the fixed-budget successive-rejects guarantee of \citet{audibert2010bestarm} to these arms with $B_{smooth}/\cp$ observations gives
$$\mathbb{P}(E_{smooth}^c \mid E_{det}) \le O(|\mathcal{U}|) \exp\left(-\kappa \frac{B_{smooth}}{c_0}\right),$$
with $c_0$ as defined above and $\kappa$ universal. The descent that produces $\mathcal{U}$ uses sound pruning only, and on $E_{det}$ it never discards the optimum's branch, so its cost is counted inside $B_{smooth}$ and it only shrinks $|\mathcal{U}|$.

\textbf{Step 3: Leaf Certification on the Flagged Set.} \\
If $x^*$ lies within one of the $K$ flagged jump cells, the algorithm must identify it using expensive leaf evaluations. The algorithm runs a successive elimination or uniform exploration routine restricted to the leaves of the $K$ flagged cells. 
Let $B_{cert}$ be the portion of the budget spent on these expensive evaluations. Because the $K$ cells are disjoint, the hardness of identifying $x^*$ among them is exactly the sum of their individual hardnesses. The probability of returning a suboptimal leaf from the flagged set is bounded by standard unstructured best-arm identification guarantees \citep{audibert2010bestarm}:
$$\mathbb{P}(E_{cert}^c \mid E_{det}) \le O(K) \exp\left(-\kappa \frac{B_{cert}}{\sum_{k=1}^K \cl h_k}\right)$$

\textbf{Step 4: Budget Allocation and Synthesis.} \\
The total budget is $B = B_{smooth} + B_{cert}$. The algorithm statically or dynamically allocates the budget proportional to the empirical complexities of the respective phases. To minimize the maximum of the two error probabilities, the budget splits such that the exponents are equalized:
$$\frac{B_{smooth}}{c_0} = \frac{B_{cert}}{\sum_{k} \cl h_k} = \frac{B}{c_0 + \sum_k \cl h_k} = \frac{B}{\Hedge}$$

Applying a union bound over the failure modes, the total probability of misidentification is:
$$\mathbb{P}(\widehat{x} \ne x^*) \le \mathbb{P}(E_{det}^c) + \mathbb{P}(E_{smooth}^c \mid E_{det}) + \mathbb{P}(E_{cert}^c \mid E_{det})$$
Substituting the bounds derived in Steps 1 through 3:
$$\mathbb{P}(\widehat{x} \ne x^*) \le \delta_{\mathrm{det}} + O(|V|) \exp\left(-\kappa \frac{B}{\Hedge}\right) + O(K) \exp\left(-\kappa \frac{B}{\Hedge}\right)$$
Because $K \le |V|$ (the number of jump cells cannot exceed the number of explored nodes), we absorb the $O(K)$ term into the $O(|V|)$ leading factor, yielding the final bound:
$$\mathbb{P}(\widehat{x} \ne x^*) \le \delta_{\mathrm{det}} + O(|V|) \exp\left(-\kappa \frac{B}{\Hedge}\right)$$
\end{proof}

\subsection{Proof of the Partial Lower Bound (\autoref{prop:lower})}

\begin{proof}
Fix a violation cell $C$ of $m$ leaves at the penultimate level, and consider the instance
family $\{\mu_\pi\}$ indexed by permutations $\pi$ of a fixed gap profile
$(\Delta_1,\dots,\Delta_m)$ within $C$: leaf $x_i$ of $C$ has mean $\mu^* - \Delta_{\pi(i)}$
(with $\Delta_{\pi^{-1}(1)}=0$ the unique optimum), and every leaf outside $C$ has a common
constant mean strictly below $\mu^* - \max_i \Delta_i$. All instances in the family share
identical leaf means outside $C$ and an identical multiset of means inside $C$.

\textbf{Step 1: Cheap probes carry no information about the optimum's identity.} \\
Consider any internal-node probe, whether it returns the subtree average directly or a
random-path sample. A direct-average probe of a node $v$ returns $f(v)+\eta$. If
$C \subseteq v$ then $f(v)$ depends only on the sum of the means in $C$, which is
permutation-invariant, and if $v \cap C = \emptyset$ it does not depend on $\pi$ at all. No
other case exists, since $C$ sits at the penultimate level and its only descendants are
leaves, which are expensive evaluations by definition. A random-path probe of $v$ returns
$\mu(L)+\eta$ for $L$ uniform on the leaves under $v$, and its law is the uniform mixture
over those leaves' means, again a function of the multiset only. Hence the joint
distribution of any sequence of cheap probes is the same for every $\pi$, and cheap
observations are ancillary for identifying the $\arg\max$ within $C$.

\textbf{Step 2: Reduction to leaf-only identification.} \\
Any algorithm $\mathcal{A}$ with cost budget $B$ makes at most
$n_\ell = \lfloor B/\cl \rfloor$ leaf evaluations. Conditioning on the ancillary cheap
observations, $\mathcal{A}$ induces a leaf-only fixed-budget identification algorithm over
the $m$ arms of $C$ with budget $n_\ell$. Leaf evaluations outside $C$ are known constants
and carry no information about $\pi$, and an algorithm that randomizes using the ancillary
observations cannot decrease the minimax error. By the fixed-budget lower bound of
\citet{carpentier2016tight}, for every such algorithm there is a permutation $\pi$ under
which
$$\Prob(\widehat x \ne x^*) \;\ge\; c'\exp\!\left(-\,\frac{c''\, n_\ell \log m}{h}\right),
\qquad h=\sum_{i}\frac{\sigma^2}{\Delta_i^2},$$
which is a constant whenever $n_\ell \le c\, h/\log m$, that is, whenever
$B \le c\,\cl h /\log m$.

\textbf{Step 3: What the bound does and does not show.} \\
On this family $\Hedge = c_0 + \cl h$, and the certification term dominates because the
smooth part is flat. The certification component of $\Hedge$ is therefore necessary up to
the $\log m$ factor. The argument says nothing about the structural component $c_0$, which
is negligible on this family by construction, and that is why the lower bound is partial.
\end{proof}

\section{Regret Analysis in the Piecewise-Lipschitz Regime}
\label{app:regret}

This appendix proves the regret bound for the optimistic tree bandit on a piecewise-Lipschitz objective. Inflating the optimism bonus only on the detected discontinuity cells makes the regret decompose into the smooth rate plus an additive penalty in the number of discontinuities.

\subsection{The Tree-Dispersion Condition}

Let $\mathcal{X}$ be the leaf space with the LCA distance $d(x,y) = \rho^{\text{level}(\text{LCA}(x,y))}$ of \autoref{sec:setup}. A function $\mu: \mathcal{X} \to [0,1]$ is globally tree-Lipschitz with constant $L$ if $|\mu(x) - \mu(y)| \le L d(x,y)$ for all $x,y \in \mathcal{X}$.

In our setting $\mu$ is only piecewise tree-Lipschitz. We define this through the dispersion condition of \citet{balcan2018dispersion}, adapted to the tree metric.

\begin{definition}[Tree-Dispersion]
A function $\mu$ is piecewise tree-Lipschitz with $K$ discontinuities if there exists a set of boundary points $\mathcal{D} \subset \mathcal{X}$ with $|\mathcal{D}| \le K$ such that for any cell $\mathcal{C}(v)$ that does not contain a point in $\mathcal{D}$, the restriction of $\mu$ to $\mathcal{C}(v)$ is tree-Lipschitz.
\end{definition}

Consequently, at any level $l \in \{1, \dots, D\}$, at most $K$ cells straddle a discontinuity. Summing over all depths, the total number of ``jump cells'' in the entire tree is bounded by $K \cdot D$.

\subsection{Formal Proof of \autoref{thm:regret}}

\begin{theorem}[Regret with $K$ Discontinuities]
\label{thm:app_regret}
Assume $\mu$ is piecewise tree-Lipschitz with $K$ dispersed discontinuities of height at most 1, and the smooth regions have a near-optimality dimension $d$. Let the detection mechanism identify the $O(KD)$ jump cells with probability $1 - \delta_{\mathrm{det}}$, and let $\Delta_{\min}$ be the smallest positive sub-optimality gap among jump cells. Conditioned on successful detection, the cumulative regret $R_n$ is bounded by:
$$R_n \le C_1 n^{\frac{d+1}{d+2}} + C_2 \frac{K D \log n}{\Delta_{\min}}$$
where $C_1$ depends on the Lipschitz constant $L$, discount $\rho$, and dimension $d$, and $C_2$ is a universal constant.
\end{theorem}

\begin{proof}
Let $\mathcal{V}$ be the set of all nodes in the tree. We partition $\mathcal{V}$ into two disjoint sets: the smooth cells $\mathcal{V}_{smooth}$ and the jump cells $\mathcal{V}_{jump}$. By the tree-dispersion condition, $|\mathcal{V}_{jump}| \le K \cdot D$. 

The cumulative regret over $n$ steps can be decomposed based on the type of cell the algorithm commits to. Consistently with the definition in \autoref{sec:setup}, regret is measured against the value of the committed cell $v_t$ (a random-path probe of $v_t$ has expectation $f(v_t)$, so this coincides with the expected sampled-leaf regret):
$$R_n = \sum_{t=1}^n (\mu^* - f(v_t)) = \sum_{t \in \mathcal{T}_{smooth}} (\mu^* - f(v_t)) + \sum_{t \in \mathcal{T}_{jump}} (\mu^* - f(v_t))$$
where $\mathcal{T}_{smooth}$ (resp. $\mathcal{T}_{jump}$) is the set of rounds where the committed cell $v_t$ was a smooth (resp. jump) cell.

\textbf{Step 1: Bounding Regret on Smooth Cells.} \\
For a node $v \in \mathcal{V}_{smooth}$ the subtree contains no discontinuity, so the aggregation bias $\bias(v)$ is bounded by the schedule $\spread(l) = L \rho^l$. On the successful detection event, the optimistic index $U(v)$ is therefore a valid upper bound for every $v \in \mathcal{V}_{smooth}$.

Because the confidence bounds are valid and the smoothness schedule holds on this subset, the standard HOO analysis of \citet{bubeck2011xarmed} applies unchanged. The number of pulls allocated to suboptimal smooth cells is controlled by the near-optimality dimension $d$. Summing the regret over $\mathcal{T}_{smooth}$ yields the classical rate
$$\sum_{t \in \mathcal{T}_{smooth}} (\mu^* - f(v_t)) \le C_1 n^{\frac{d+1}{d+2}}$$
where $C_1$ absorbs the logarithmic terms and dependencies on $L, \rho$, and $d$.

\textbf{Step 2: Bounding the Discontinuity Penalty.} \\
For a node $v \in \mathcal{V}_{jump}$, the local variation may violate $\spread(l)$. However, the algorithm identifies these cells and inflates their bias bonus to the empirical worst-case bound (\autoref{app:certification}), so the inflated index $U(v)$ is a valid upper confidence bound on the best leaf under $v$.

Fix a suboptimal jump cell $v$ with sub-optimality gap $\Delta_v = \mu^* - \max_{x \in \mathcal{C}(v)} \mu(x) > 0$. With the inflated and valid bonus, the standard UCB argument applies to $v$ as to a single arm with reward range $[0,1]$. After $T_v(t)$ visits its confidence radius is $O(\sqrt{\log t / T_v(t)})$, and once $T_v(t) \ge c \log n / \Delta_v^2$ for a universal $c$ the optimistic index of $v$ falls below $\mu^*$ with high probability, so $v$ is no longer selected. Each visit to $v$ incurs instantaneous regret at most $\Delta_v + \spread(l_v)$, which is $O(\Delta_v)$ once the bonus is valid, since the cell's internal spread is then dominated by its gap and the discontinuity height of at most $1$ bounds the worst case. The aggregate regret charged to $v$ over the horizon is therefore
$$C_v \;\le\; O\!\left(\frac{\log n}{\Delta_v^2}\right) \cdot O(\Delta_v) \;=\; O\!\left(\frac{\log n}{\Delta_v}\right).$$

By the tree-dispersion condition $|\mathcal{V}_{jump}| \le K \cdot D$, so summing over jump cells,
$$\sum_{t \in \mathcal{T}_{jump}} (\mu^* - f(v_t)) \;\le\; \sum_{v \in \mathcal{V}_{jump}} C_v \;\le\; C_2 \, \frac{K D \log n}{\Delta_{\min}},$$
with $\Delta_{\min} = \min_{v \in \mathcal{V}_{jump}: \Delta_v > 0} \Delta_v$ and $C_2$ universal. (A jump cell containing the optimum has $\Delta_v = 0$ and contributes no regret.)

\textbf{Step 3: Synthesis.} \\
Combining the bounds for $\mathcal{T}_{smooth}$ and $\mathcal{T}_{jump}$ gives the decomposition
$$R_n \le C_1 n^{\frac{d+1}{d+2}} + C_2 \frac{K D \log n}{\Delta_{\min}}.$$
For any fixed $K$ the jump term grows only logarithmically in $n$, so the polynomial term dominates as $n \to \infty$. The $K$ dispersed discontinuities change constants and a $\log n$ factor, not the learning rate, and $K = 0$ recovers the clean Lipschitz rate.
\end{proof}

\section{Synthetic Validation of the Multi-Fidelity Machinery}
\label{app:synthetic}

These controlled studies test the core bandit claims under conditions a fixed real
benchmark cannot provide. We sweep the probe/leaf cost ratio, the number of
discontinuities $K$, and the assumed smoothness. Every synthetic result is averaged over
many seeds and reported with $95\%$ confidence bands.

\paragraph{The tree prior pays off only when probes are cheap.}
We run top-$5$ identification on a $1024$-leaf tree over $60$ seeds against a strong
structure-blind successive-elimination baseline (\autoref{fig:fidelity}). When an
internal probe costs as much as a leaf, the descent is pure overhead and the tree loses.
Once probes are at least $2\times$ cheaper the tree overtakes the baseline, and it
approaches full recall as probes get cheaper still. \autoref{fig:fidelity} plots two
hierarchical variants to show what carries the win. The practical beam variant,
best-first focusing with beam width $20$, overtakes successive elimination once probes
are cheap, $0.84$ against $0.57$ at budget $1500$. The sound variant runs exactly the
worst-case-valid pruning that \autoref{thm:fixed-budget} analyzes, and on this smooth
family it only matches the baseline, $0.62$ against $0.57$, because the conservative
spread bound prunes little. The empirical advantage therefore comes from the beam
focusing. The sound algorithm is the certified never-worse floor, and sharper certified
pruning remains open.

\paragraph{Regret versus memory.}
We run online regret minimization on a $256$-leaf tree of depth $4$ over a $12$k-round
horizon and $16$ seeds (\autoref{fig:memory}). Fixed-resolution play traces a clean
frontier, in which deeper trees have less bias but exponentially more memory. The
expand-versus-refine rule reaches near full-resolution regret at about a third of the
memory. The variance-aware variant, which estimates the bias from data rather than
assuming $\spread$, reaches it at about $6\times$ less and sits below the fixed-depth
frontier.

\paragraph{Graceful degradation in the violation count.}
We measure top-$1$ accuracy at a fixed tight budget as the number of discontinuities $K$
grows, with branching $4$, depth $5$, budget $400$, $\cp/\cl=0.05$, and $30$ seeds
(\autoref{fig:violations}). The data-driven hybrid degrades smoothly from near-perfect
accuracy when violations are few toward the structure-blind floor as they proliferate,
and it strictly dominates both alternatives. Pure assumed smoothness is misled by the
spikes, and the blind baseline cannot afford enough leaf evaluations. This is the
empirical form of $\Hedge=c_0+\sum_{k\le K}\cl h_k$ saturating to $\Hblind$.

\paragraph{Finite-state compression.}
On the infinite-depth tree, with arity $3$, a $15$k horizon, and the memory-bounded rule,
the tree grows from $121$ to nearly $800{,}000$ nodes while the explored memory stays
near $70$ and the regret stays between $663$ and $1702$. This is the discrete witness of
the depth-independent memory in \autoref{prop:regret}.

\paragraph{Estimating the smoothness.}
The optimistic descent is inherited, so this ablation isolates the paper's contribution,
estimating the local Lipschitz constant from data rather than assuming a schedule. The
instance is a heterogeneous-smoothness tree whose left half is smooth and whose right
half is rough, so no single global constant fits. Sweeping the assumed constant $L$
traces a U-shaped tuning curve. A constant that is too tight under-explores the rough
half and misses the optimum, and a constant that is too loose over-explores everywhere.
The data-driven per-subtree estimator of \autoref{sec:smoothness} uses no tuning at all
(\autoref{fig:ablation}, \autoref{tab:ablation}). It matches and slightly
beats the oracle-tuned fixed constant, $3029$ against $3127$ final regret at the swept
minimum, and it far outperforms a mis-specified tight constant at $4076$, because it
adapts a local constant per subtree that no global $L$ can match. The win comes from the
certified and estimated smoothness, not from the inherited optimistic descent.

\paragraph{From the rate to the data.}
Two checks connect the regret bound of \autoref{sec:theory} to the experiments
(\autoref{fig:theorylink}). First, on a tree-Lipschitz value function the average regret
$R(n)/n$ decays toward zero as the horizon grows. The measured cumulative-regret
exponent is $\alpha\approx0.77<1$, with the near-optimality dimension estimated from the
value function at $\hat d\approx0$, which is the empirical form of the sublinear
$\Otil(n^{(d+1)/(d+2)})$ rate. Second, we ask whether the prior holds on real routing
data. A variance decomposition of the RouterBench value
function~\citep{hu2024routerbench}, over $36$k prompts, $11$ models, and $86$
categories, shows that the region tree captures the exploitable regional signal. The
category tree explains about $31\%$ of the per-prompt value variance, and the
benchmark-family tree explains about $53\%$ of the region-value variance. The large
remaining within-region variation is per-prompt difficulty, observation noise the bandit
averages over. The learner needs only the regional optimum, not the full noisy
per-prompt value function (\autoref{sec:setup}), so the routing value function is
tree-Lipschitz at the regional scale rather than finely smooth. The finer
almost-tree-$K$-Lipschitz characterization, in which a partial trace's cheap value
predicts its completion and the violations are the pivotal decision steps, lives in the
reasoning traces and is reported in \autoref{sec:exp-reasoning}.

\begin{figure}[t]
\centering
\includegraphics[width=\linewidth]{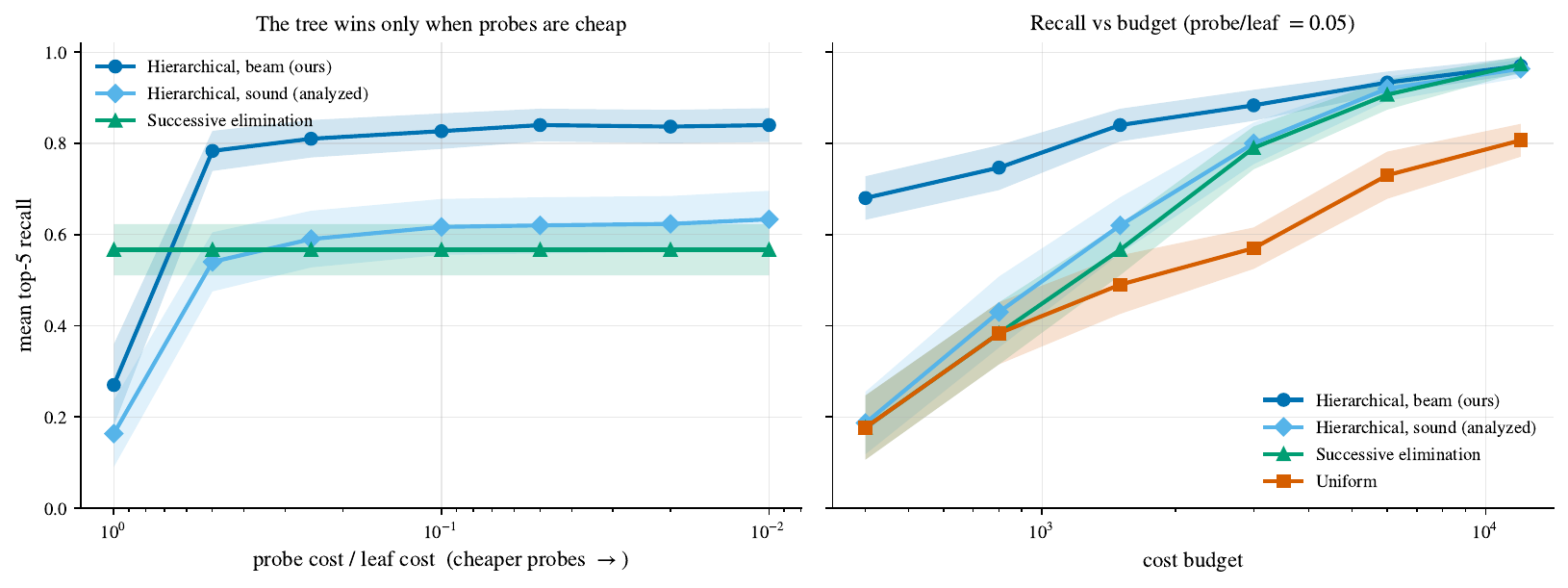}
\caption{Multi-fidelity top-$k$ on a $1024$-leaf tree ($60$ seeds, 95\% CI bands). Left:
recall vs.\ probe/leaf cost ratio at a fixed budget. Right: recall vs.\ cost budget with cheap
probes, for the beam and sound hierarchical variants against successive elimination.}
\label{fig:fidelity}
\end{figure}

\begin{figure}[t]
\centering
\includegraphics[width=\linewidth]{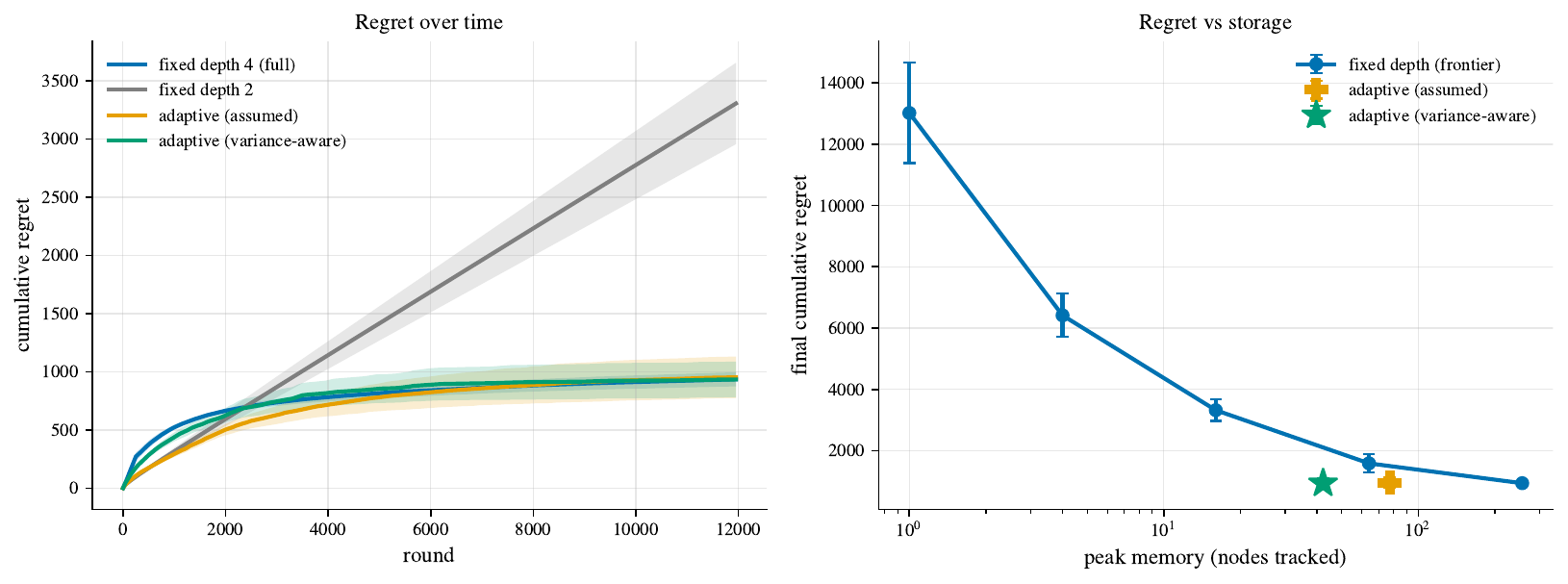}
\caption{Regret vs.\ storage, isolated on controlled instances ($256$ leaves, depth $4$,
horizon $12$k, $16$ seeds). The data-driven expand-vs-refine rule reaches near
full-resolution regret at ${\approx}6\times$ less memory than the fixed-depth frontier.}
\label{fig:memory}
\end{figure}

\begin{figure}[t]
\centering
\figorpending{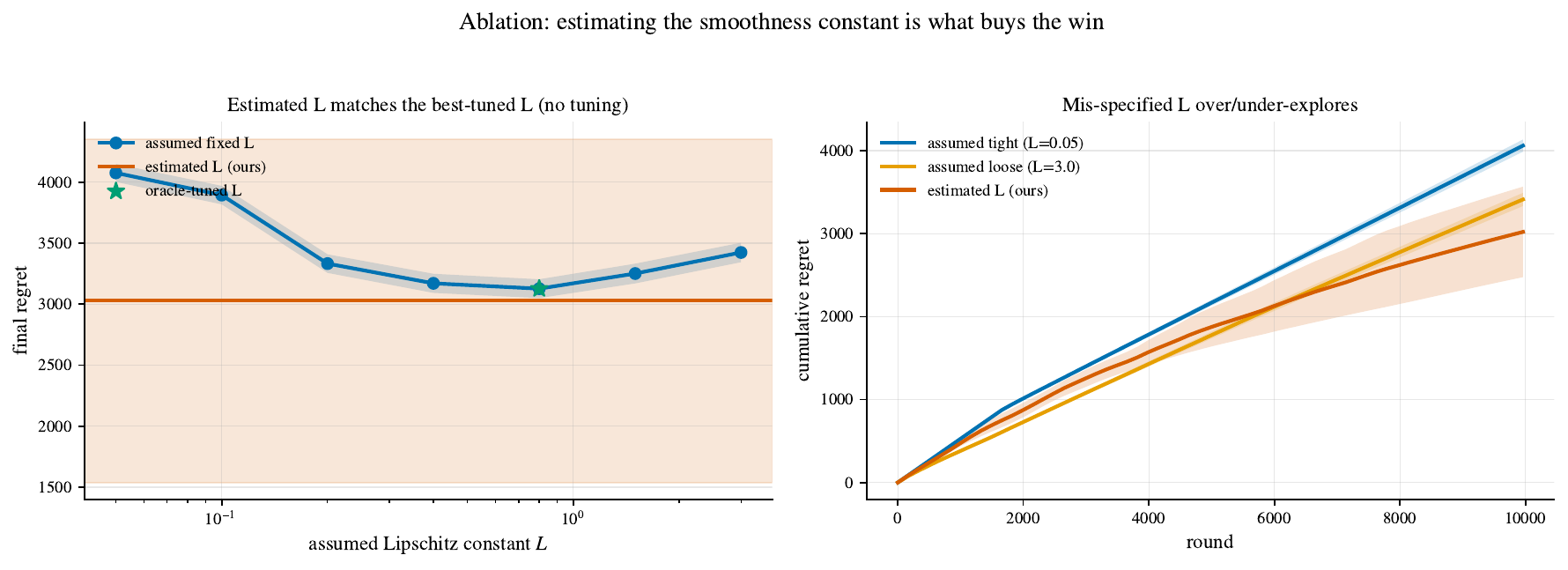}{examples/tree\_bandits/lipschitz\_ablation.py}
\caption{The Lipschitz-constant ablation on a heterogeneous-smoothness tree ($10$ seeds).
The assumed-$L$ tuning curve is U-shaped, while the estimated-$L$ policy (ours) matches the
oracle-tuned minimum with no tuning ($3029$ vs.\ $3127$ final regret).}
\label{fig:ablation}
\end{figure}

\begin{figure}[t]
\centering
\includegraphics[width=\linewidth]{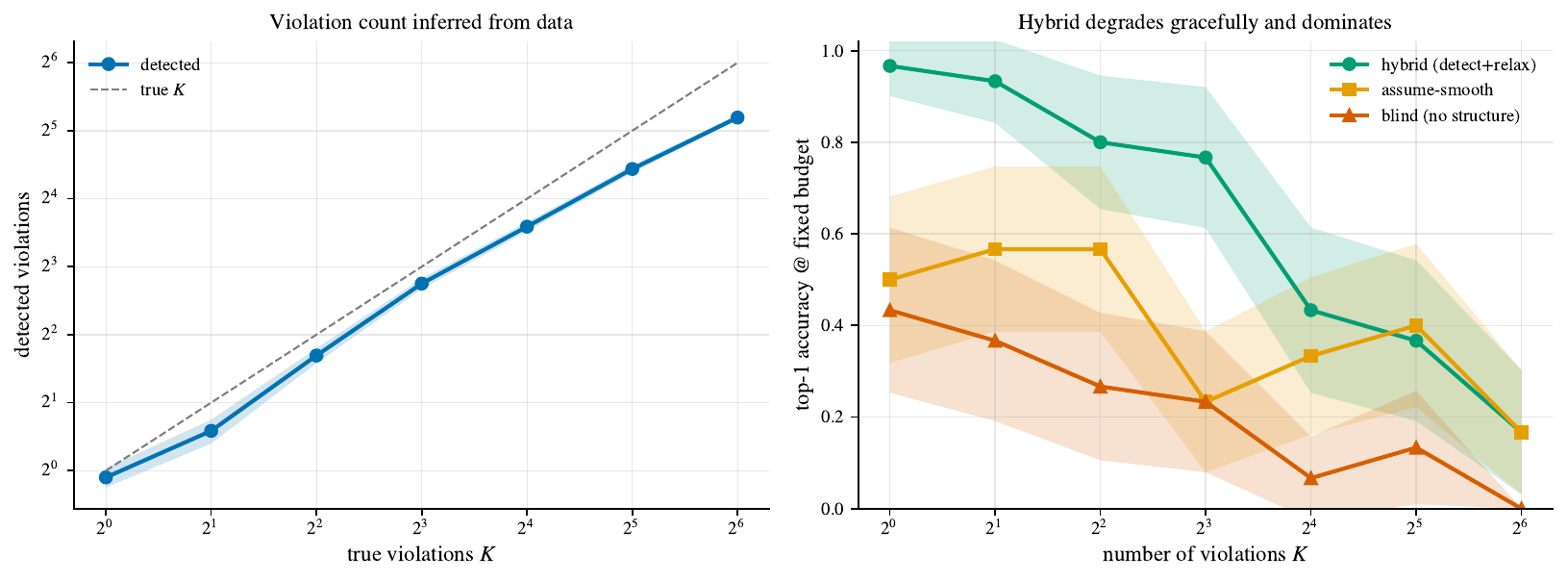}
\caption{Value of structure vs.\ number of Lipschitz violations $K$ ($30$ seeds, 95\% CI
bands). Left: the violation count is recovered from data. Right: top-$1$ accuracy at a fixed
budget, where the detect-and-relax hybrid degrades gracefully and dominates both baselines.}
\label{fig:violations}
\end{figure}

\begin{table}[t]
\centering
\small
  \begin{tabular}{lr}
\toprule
Policy & Final regret \\
\midrule
Assumed L (too tight, $L=0.05$) & 4076 \\
Assumed L (too loose, $L=3.0$) & 3424 \\
Oracle-tuned fixed $L=0.8$ & 3127 \\
\textbf{Estimated L (ours)} & \textbf{3029} \\
\bottomrule
\end{tabular}

\caption{Lipschitz-constant ablation: final regret for mis-specified fixed $L$, the oracle-tuned
fixed $L$, and the data-driven estimator (ours).}
\label{tab:ablation}
\end{table}

\begin{figure}[t]
\centering
\figorpending[\linewidth]{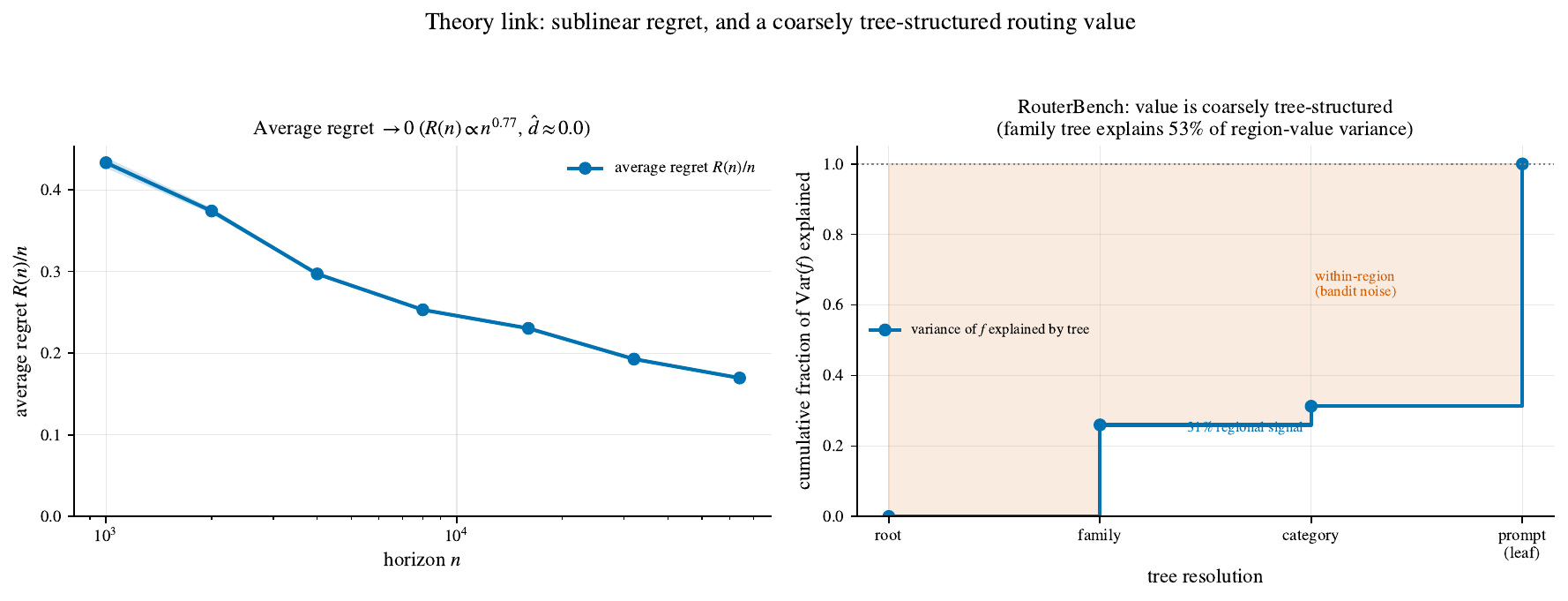}{examples/analysis/theory\_link.py}
\caption{Linking theory to data ($8$ seeds). Left: on a tree-Lipschitz value function the average
regret $R(n)/n$ decays toward zero ($R(n)\propto n^{\alpha}$, $\alpha\approx0.77<1$;
$\hat d\approx0$). Right: variance decomposition of the real RouterBench routing value
function. The region tree captures the coarse regional signal, and the within-region remainder
is per-prompt noise the bandit averages over.}
\label{fig:theorylink}
\end{figure}

\section{Additional Experimental Results}
\label{app:experiments}

This appendix collects the per-benchmark tables, secondary figures, and negative results
behind \autoref{sec:experiments}. 

\subsection{Routing}

\begin{figure}[t]
\centering
\begin{minipage}[t]{0.48\linewidth}
\centering
\includegraphics[width=\linewidth]{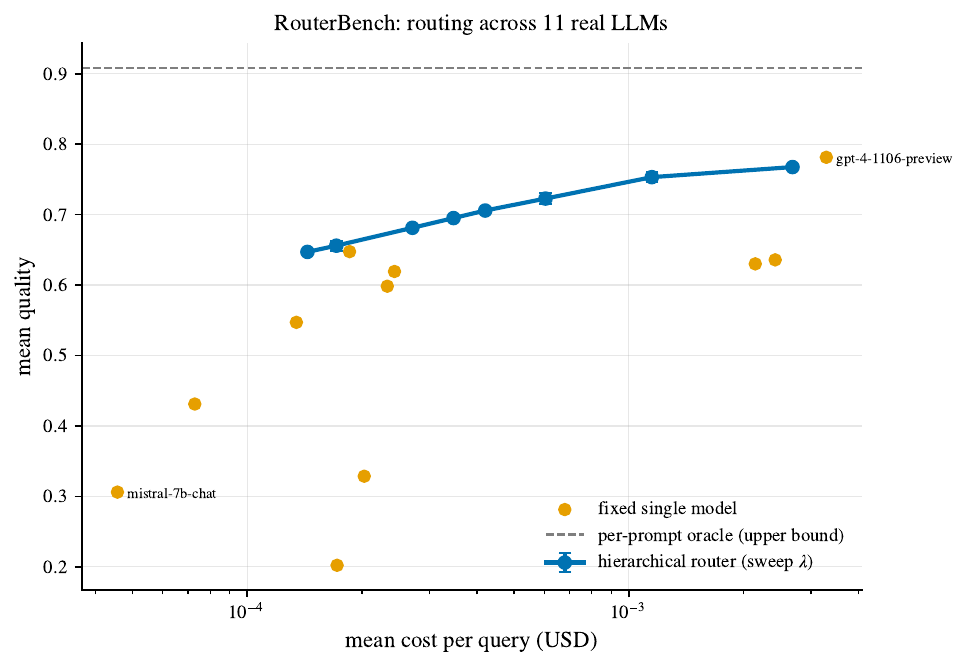}
\end{minipage}\hfill
\begin{minipage}[t]{0.48\linewidth}
\centering
\figorpending{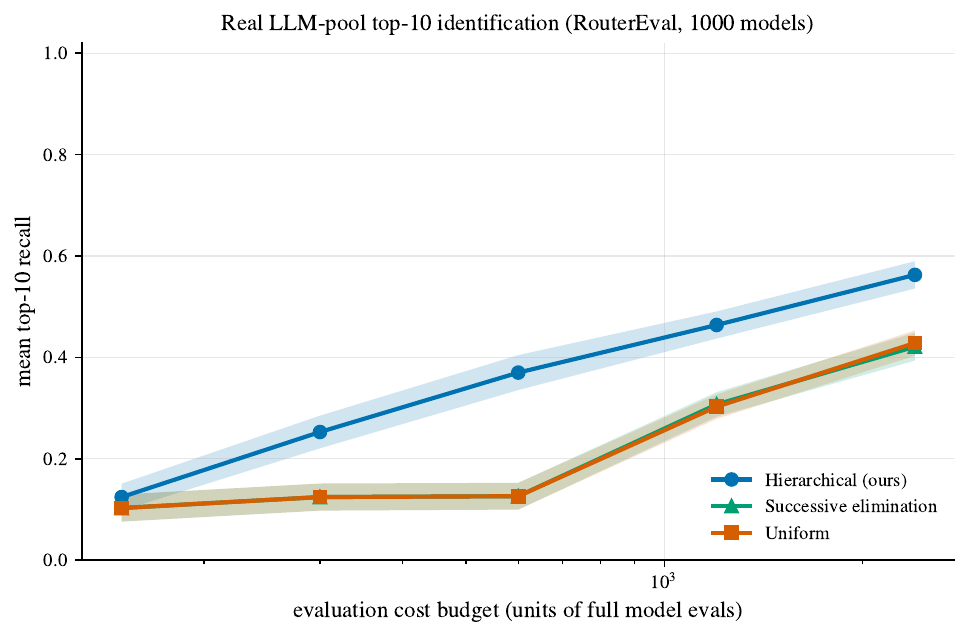}{examples/tree\_bandits/real\_topk\_identification.py}
\end{minipage}
\caption{Structure against structure-blind baselines on real model pools. Left: RouterBench
routing. Sweeping the cost weight $\lambda$ traces a router frontier (blue) that dominates
every fixed model (orange), and the dashed line is the per-prompt oracle. Right: top-$10$
recall vs.\ evaluation cost on RouterEval's $1000$-model pool ($9$ evals $\times$ $20$ seeds,
$95\%$ CI bands), where the tree bandit dominates both structure-blind baselines at every
budget.}
\label{fig:routerbench}\label{fig:realtopk}
\end{figure}

\begin{figure}[t]
\centering
\includegraphics[width=\linewidth]{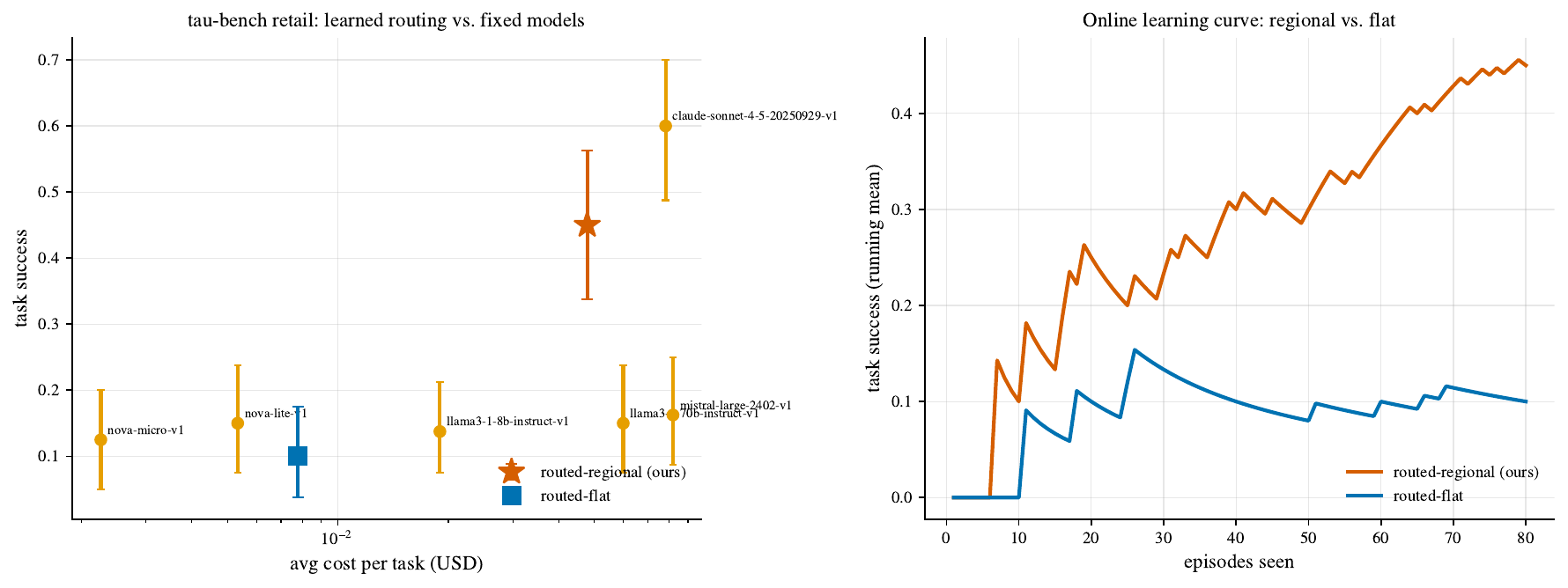}
\caption{Per-call routing inside long-horizon $\tau$-bench agents. Left: retail
cost/quality, where the learned regional router (red star) reaches the fixed-model cloud at
low cost. Right: the online curve separates the regional from the flat learner.}
\label{fig:taubench}
\end{figure}

\paragraph{Live models on MMLU.}
We route six Bedrock models spanning roughly a $100\times$ cost gradient, with MMLU
subjects as regions. The regional router reaches near-oracle quality at the lowest cost
of any realizable policy, and only the truth-based per-prompt oracle is cheaper. It beats
both the structure-blind flat learner and the strongest fixed model (\autoref{tab:mmlu}).

\paragraph{Per-call routing inside long-horizon agents.}
On $\tau$-bench the region is a small discrete difficulty-and-phase feature of the turn,
the arm is a model, and the quality signal is the episode's terminal reward, credited to
every region-model pull along the episode. A single trial's bootstrap interval captures
task-to-task spread but not seed variance, so we ran $5$ independent replicates of the
two learned routers over the same $80$ tasks (\autoref{tab:taureps}). The regional
learner beats the flat learner in $4$ of $5$ replicates, with mean success
$0.463 \pm 0.156$ against $0.323 \pm 0.213$ and gap $+0.14 \pm 0.27$. The variance is
dominated by which models each learner converges to under terminal-reward credit
assignment. In the one reversed replicate the flat learner drifted to expensive models
and the regional learner did not. We therefore present $\tau$-bench as directional
evidence and let RouterBench and live MMLU carry the routing claim. For completeness, the
strongest fixed policy, always claude-sonnet-4.5, attains higher raw success at $0.600$
than either learned router, at roughly $1.6\times$ the regional router's realized cost
per task (\autoref{tab:taubench}). The learned routers' advantage on this benchmark is
therefore cost-adjusted, not absolute. A per-turn value signal would be tighter and is
future work.

\begin{table}[!t]
\centering
\small
\textbf{\textsc{MMLU: live Bedrock routing}}\par\smallskip
\tableorpending{figures/mmlu_routing_table.tex}{examples/llm\_routing/mmlu\_routing.py}
\par\medskip
\textbf{\textsc{$\tau$-bench retail: per-call routing in long-horizon agents}}\par\smallskip
\emph{Per-policy results (main run)}\par\smallskip
\tableorpending{figures/taubench_routing_table.tex}{examples/agentic/taubench\_routing.py}
\par\medskip
\emph{Seed-variance study ($5$ independent replicates)}\par\smallskip
\tableorpending{figures/taubench_replicates_table.tex}{examples/agentic/combine\_taubench\_replicates.py}
\caption{Routing detail. Top: routing over six real Bedrock models on MMLU, with $8$
subjects as regions and a horizon of $6$k. Only the hierarchical and flat policies learn
online. The rest are computed from the ground truth and pay no exploration cost. Middle:
$\tau$-bench retail, the online contextual UCB router (regional) against the
structure-blind flat learner and the fixed models. Bottom: the seed-variance study, $5$
independent replicates of the two learned routers over the same $80$ tasks, of which
replicate $1$ is the main run above.}
\label{tab:mmlu}\label{tab:taubench}\label{tab:taureps}
\end{table}

\subsection{Top-$k$ identification}

\autoref{tab:realtopk} reports the full budget sweep behind the identification row of
\autoref{tab:summary} and \autoref{fig:realtopk}. The pool is RouterEval's $1000$
models arranged as a branching-$10$, depth-$3$ tree, the metric is mean recall@$10$ over
$20$ seeds and $9$ evaluation splits, and a probe costs $0.05$ of a leaf evaluation. The
charged pre-pass that measures the spread is included in our method's cost. Two patterns
in the sweep support the theory. First, successive elimination never separates from
uniform sampling, staying within one standard error of it at every budget, so the win
comes from the tree structure rather than from an elimination schedule. Second, the
hierarchical gain is largest at low-to-mid budgets, $2.9\times$ recall at $B{=}600$, and
narrows to $1.3\times$ at $B{=}2400$ as the structure-blind baselines begin to catch up.

\begin{table}[t]
\centering
\footnotesize
\setlength{\tabcolsep}{4pt}
\begin{tabular}{lccccc}
\toprule
Cost budget & 150 & 300 & 600 & 1200 & 2400 \\
\midrule
\textbf{Hierarchical (ours)} & \textbf{0.124\,$\pm$\,0.013} & \textbf{0.253\,$\pm$\,0.016} & \textbf{0.370\,$\pm$\,0.018} & \textbf{0.464\,$\pm$\,0.014} & \textbf{0.563\,$\pm$\,0.014} \\
Successive elimination & 0.103\,$\pm$\,0.014 & 0.124\,$\pm$\,0.014 & 0.126\,$\pm$\,0.013 & 0.307\,$\pm$\,0.012 & 0.421\,$\pm$\,0.014 \\
Uniform & 0.103\,$\pm$\,0.014 & 0.124\,$\pm$\,0.014 & 0.126\,$\pm$\,0.013 & 0.303\,$\pm$\,0.012 & 0.428\,$\pm$\,0.013 \\
\bottomrule
\end{tabular}
\caption{Top-$10$ identification on the RouterEval $1000$-model pool: mean recall@$10$
($\pm$ SEM over $20$ seeds and $9$ evaluation splits) as the cost budget grows. The
probe/leaf cost ratio is $0.05$, and the hierarchical method's cost includes the
certificate pre-pass.}
\label{tab:realtopk}
\end{table}

\subsection{Test-time search}

This subsection collects the accuracy-vs-budget curves, per-benchmark tables, the capability
ladder, and the tree-Lipschitz instrumentation behind the search and measured-prior
experiments.

\begin{figure}[t]
\centering
\begin{minipage}{0.49\linewidth}
\centering
\figorpending{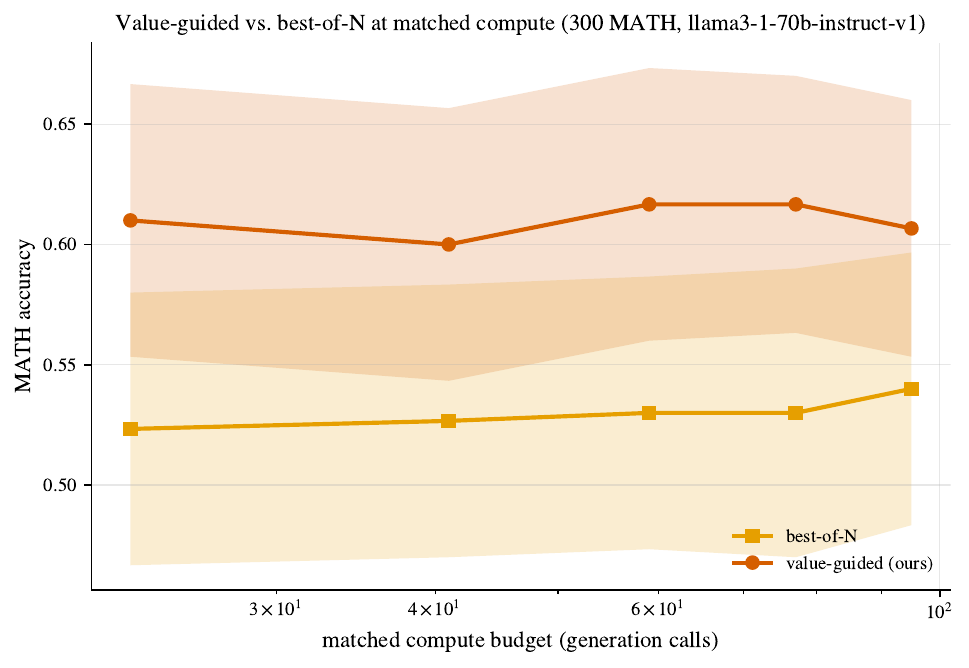}{examples/reasoning/reasoning\_search.py}
\end{minipage}\hfill
\begin{minipage}{0.49\linewidth}
\centering
\figorpending{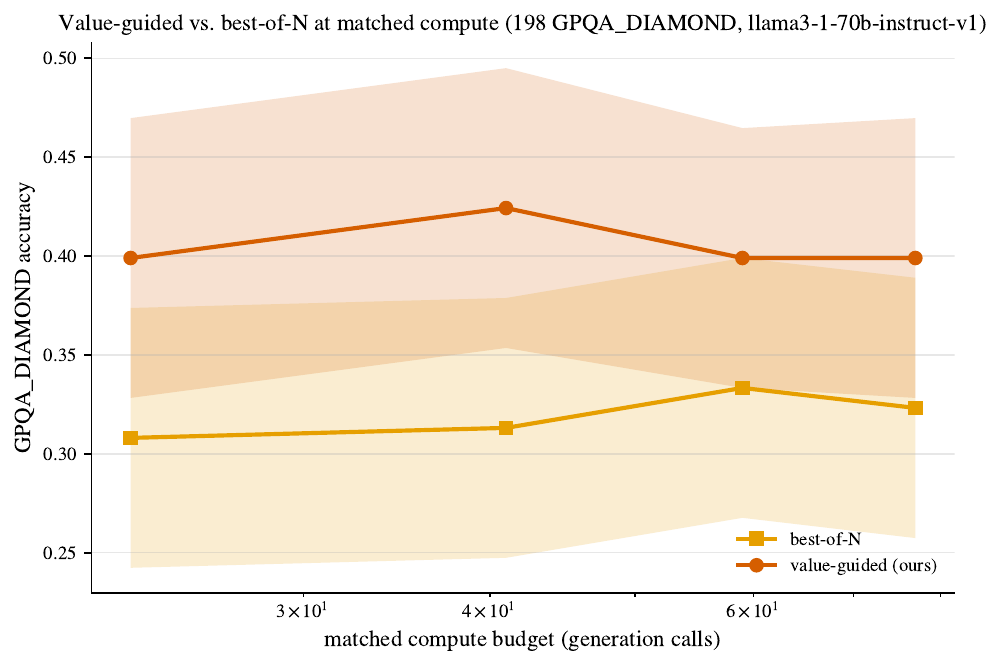}{examples/reasoning/reasoning\_search.py}
\end{minipage}
\caption{Accuracy vs.\ matched compute budget (generation calls) for value-guided search vs.\
best-of-N, with bootstrap 95\% CI bands (llama-3.1-70B). Left: MATH. Right: GPQA-Diamond.}
\label{fig:reasoning}\label{fig:reasoninggpqa}
\end{figure}

\begin{table}[t]
\centering
\small
\resizebox{\linewidth}{!}{  \begin{tabular}{lccc}
\toprule
Budget (calls) & best-of-N resolved [95\% CI] & value-guided resolved [95\% CI] & $\Delta$ (paired) [95\% CI] \\
\midrule
9 & 0.195 [0.15, 0.25] & 0.318 [0.26, 0.38] & $+0.123$ [+0.08, +0.17] \\
\bottomrule
\end{tabular}
}
\caption{Repository-level code: value-guided patch search vs.\ best-of-N at matched compute on
SWE-bench Verified ($261$ issues, ``$15$ min--$1$ hour'' tier, claude-sonnet-4.5). Resolved rate is
the official criterion, and $\Delta$ is the paired same-issue gap with a $95\%$ bootstrap CI
over issues.}
\label{tab:swebench}
\end{table}

\begin{table}[p]
\centering
\footnotesize
\textbf{\textsc{Capability ladder}}\par\smallskip
\emph{MATH}\par\smallskip
\tableorpending{figures/reasoning_models_math_table.tex}{examples/reasoning/combine\_reasoning\_models.py}
\par\medskip
\emph{GPQA-Diamond}\par\smallskip
\tableorpending{figures/reasoning_models_gpqa_diamond_table.tex}{examples/reasoning/combine\_reasoning\_models.py}
\par\medskip
\emph{GSM8K (near-saturated control)}\par\smallskip
\tableorpending{figures/reasoning_models_gsm8k_table.tex}{examples/reasoning/combine\_reasoning\_models.py}
\par\medskip
\textbf{\textsc{Repository-level code}}\par\smallskip
\emph{SWE-bench Verified, model $\times$ budget sweep}\par\smallskip
\tableorpending{figures/swebench_sweep_table.tex}{examples/reasoning/combine\_swebench\_sweep.py}
\par\medskip
\textbf{\textsc{Code synthesis (pre-registered nulls)}}\par\smallskip
\emph{HumanEval}\par\smallskip
  \begin{tabular}{rccc}
\toprule
Budget (calls) & best-of-N acc [95\% CI] & value-guided acc [95\% CI] & $\Delta$ (paired) [95\% CI] \\
\midrule
23 & 0.811 [0.75, 0.87] & 0.835 [0.77, 0.89] & $+0.024$ [-0.03, +0.08] \\
32 & 0.811 [0.75, 0.87] & 0.835 [0.78, 0.89] & $+0.024$ [-0.03, +0.08] \\
41 & 0.811 [0.75, 0.87] & 0.823 [0.76, 0.88] & $+0.012$ [-0.04, +0.07] \\
\bottomrule
\end{tabular}

\par\medskip
\emph{MBPP}\par\smallskip
  \begin{tabular}{rccc}
\toprule
Budget (calls) & best-of-N acc [95\% CI] & value-guided acc [95\% CI] & $\Delta$ (paired) [95\% CI] \\
\midrule
23 & 0.787 [0.72, 0.85] & 0.744 [0.68, 0.81] & $-0.043$ [-0.09, +0.00] \\
32 & 0.787 [0.73, 0.85] & 0.720 [0.65, 0.79] & $-0.067$ [-0.12, -0.02] \\
41 & 0.787 [0.73, 0.85] & 0.713 [0.64, 0.78] & $-0.073$ [-0.12, -0.03] \\
\bottomrule
\end{tabular}

\caption{Test-time search, value-guided against best-of-N at matched compute, with $95\%$
bootstrap confidence intervals, $\Delta$ the paired same-item gap, and bold where the
interval excludes zero. Capability ladder: five models on MATH and GPQA-Diamond, with
GSM8K as the near-saturated control. The gain is largest for capable but unsaturated
models and vanishes or reverses for the weakest, which is the unreachable regime.
Mistral-Large's low absolute MATH accuracy reflects answer-format non-compliance rather
than reasoning ability. Repository-level code: the SWE-bench model $\times$ budget sweep
over $78$--$80$ Verified ``$15$ min--$1$ hour'' issues per cell, with instances whose
evaluation harness errored dropped from both arms. Code synthesis: the HumanEval and MBPP
nulls for Llama-3.1-70B. The public-test cheap probe is weak and both benchmarks are
near-saturated, a limit fixed before the results were seen.}
\label{tab:reasoningmodels}\label{tab:reasoninghumaneval}\label{tab:reasoningmbpp}\label{tab:swesweep}
\end{table}

\paragraph{Pre-registered hard-tier run.}
A pre-registered follow-up tested whether the relative gain widens on the hardest tier.
The protocol was committed to the repository before the run, and the setup covers all
$42$ ``$1$--$4$ hours'' SWE-bench Verified issues with claude-sonnet-4.5 at matched
budgets $6$ and $9$, with machinery identical to the main sweep. The directional
hypothesis failed, and both arms fall to near floor. At $B{=}6$ each arm resolves $2/42$,
with $\Delta=0.000$ $[-0.07,+0.07]$. At $B{=}9$ the arms resolve $2/42$ and $3/42$, with
$\Delta=+0.024$ $[-0.05,+0.10]$, and we do not report ratios, per the pre-registered
instability rule. This run measures the unreachable regime of the scope paragraph
directly on code. When issues are beyond the model's reach at the given budget, guidance
has nothing to steer, exactly as the reachable-headroom condition predicts.

\begin{table}[t]
\centering
\small
\begin{minipage}{\linewidth}
\centering
\textbf{MATH}\par\smallskip
  \begin{tabular}{lr}
\toprule
Quantity & Value \\
\midrule
Cheap-vs-true node value (Spearman $\rho$) & 0.485 \\
Cheap-vs-true node value (Pearson $r$) & 0.427 \\
Edge-following hit rate, all steps (chance 0.33) & 0.940 \\
Edge-following hit rate, pivotal steps only & 0.734 [0.64, 0.82] ($n$=109) \\
Mean true value lost per step (edge gap) & 0.025 \\
Pivotal-step fraction & $\tau{=}0.25$: 11\%; $\tau{=}0.5$: 6\%; $\tau{=}0.75$: 4\% \\
\bottomrule
\end{tabular}

\par\bigskip
\textbf{SWE-bench Verified}\par\smallskip
  \begin{tabular}{lr}
\toprule
Quantity & Value \\
\midrule
Cheap-vs-true patch value (Spearman $\rho$) & 0.863 \\
Cheap-vs-true patch value (Pearson $r$) & 0.875 \\
Edge-following hit rate, all steps (chance 0.33) & 1.000 \\
Edge-following hit rate, pivotal steps only & 1.000 [1.00, 1.00] ($n$=18) \\
Pivotal-step fraction & $\tau{=}0.25$: 10\%; $\tau{=}0.5$: 10\%; $\tau{=}0.75$: 10\% \\
\bottomrule
\end{tabular}

\end{minipage}
\caption{The value function as almost tree-$K$-Lipschitz, measured in the wild:
cheap-vs-true node-value correlation (the tree-Lipschitz backbone) and the count of pivotal
steps $K$ (violations). Top: MATH partial traces (cheap self-consistency probe). Bottom: the
real-code analog on SWE-bench Verified ($60$ issues and $540$ candidate patches, cheap
FAIL\_TO\_PASS test probe vs.\ the official resolved criterion). The empirical
counterpart of the synthetic violation-family task (Appendix~\ref{app:synthetic}).}
\label{tab:reasoningtree}\label{tab:swetree}
\end{table}

\subsection{Prefix caching and systems}

This subsection details the serving experiments. The prompt-stream and Mooncake trace
tables come first, then the live-server vLLM measurements, which cover caching on and
off, the KV-budget frontier, and the GPU-calibrated eviction-policy comparison.

\begin{figure}[p]
\centering
\figorpending[\linewidth]{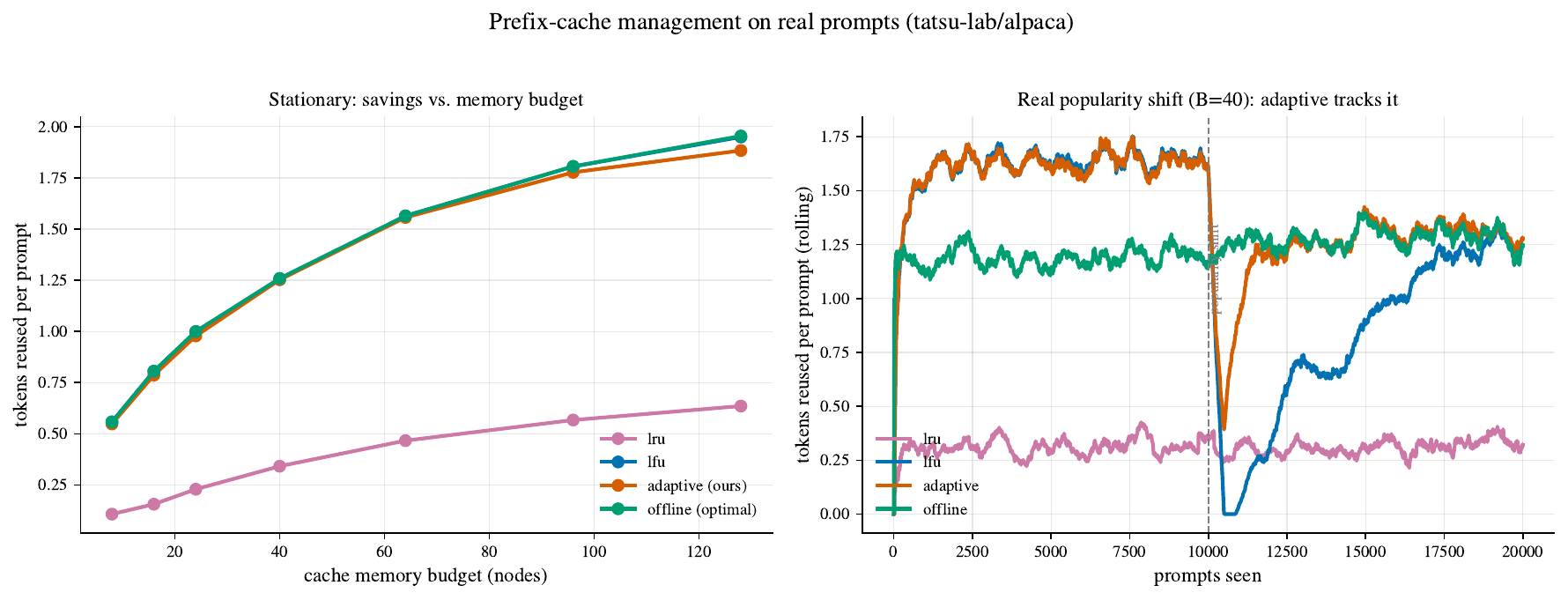}{examples/llm\_routing/prefix\_cache.py}
\caption{Prefix caching on a real prompt stream. Left: savings vs.\ memory budget
(stationary), where adaptive tracks LFU and the offline optimum. Right: savings over time
across a real popularity shift, where adaptive re-adapts fastest, LFU recovers by the end of the
stream ($1.25$ vs.\ adaptive's $1.28$ tokens/prompt at the final point), and LRU stays
stale ($0.32$).}
\label{fig:cache}
\par\medskip
\begin{minipage}[t]{0.48\linewidth}
\centering
\figorpending{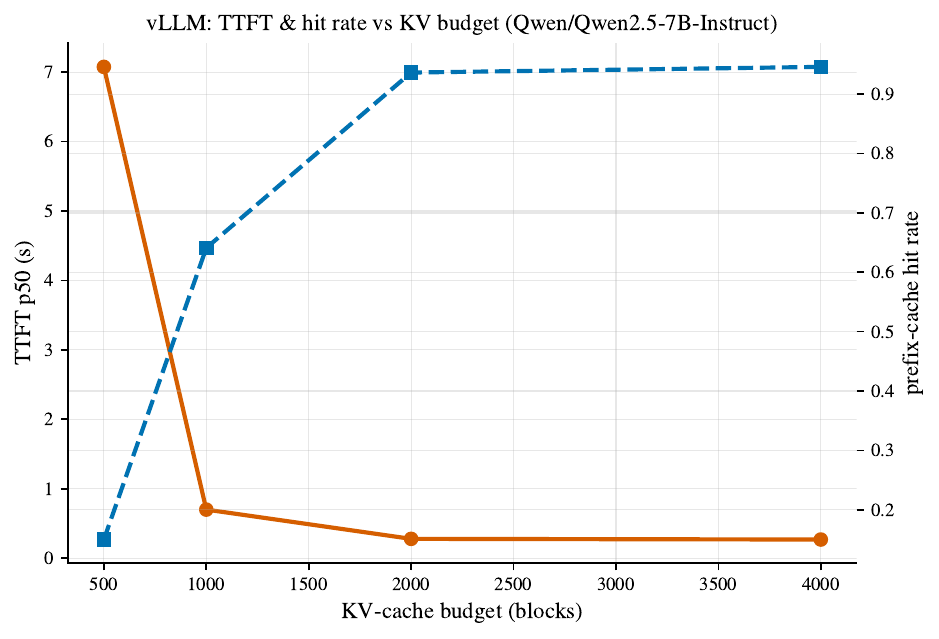}{examples/systems/vllm\_prefix\_cache\_eval.py}
\end{minipage}\hfill
\begin{minipage}[t]{0.48\linewidth}
\centering
\figorpending{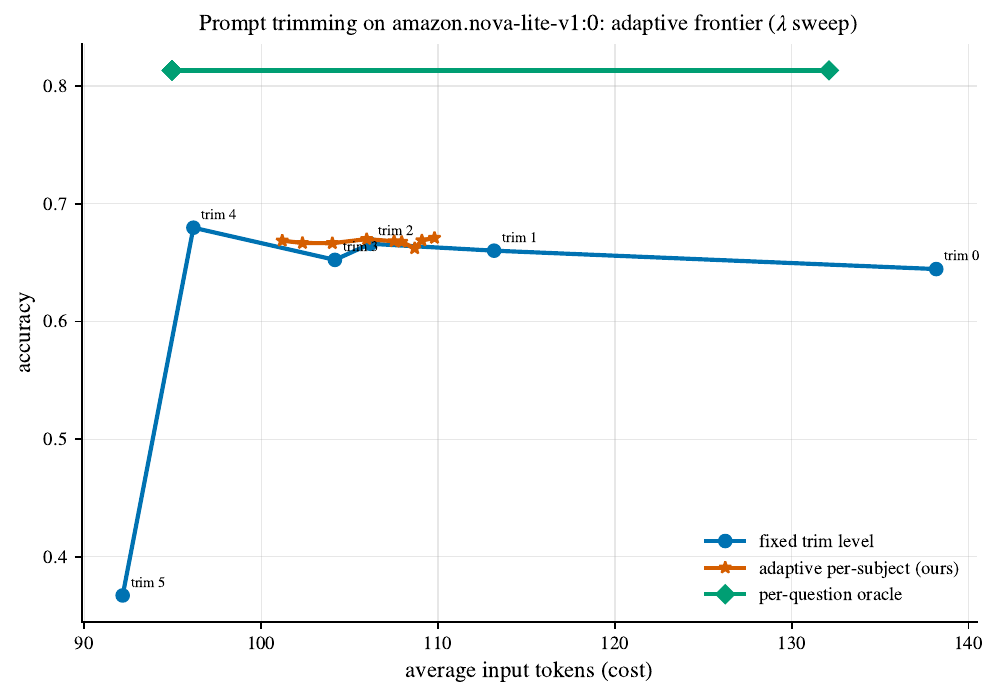}{examples/llm\_routing/prompt\_optimization.py}
\end{minipage}
\caption{Serving knobs on real systems. Left: the KV-cache budget frontier on live vLLM,
where median TTFT (log scale) falls and prefix-cache hit rate rises as the KV budget grows,
the hardware analog of savings vs.\ memory budget. Right: prompt trimming, accuracy vs.\
average input tokens. The fixed trim levels (blue) trace a tradeoff curve, sweeping $\lambda$
gives the adaptive per-subject frontier (red), which dominates the naive always-verbose
default but sits near the best fixed trim at this scale, and the per-question oracle (green)
is an upper bound.}
\label{fig:vllmbudget}\label{fig:trim}
\end{figure}

\begin{table}[p]
\centering
\footnotesize
\textbf{\textsc{Mooncake production trace}}\par\smallskip
  \begin{tabular}{lrr}
\toprule
Policy & Stationary (B=64) & After shift (B=16) \\
\midrule
LRU & 4.74 & 3.02 \\
LFU & 5.92 & 5.39 \\
\textbf{Adaptive (ours)} & \textbf{5.89} & \textbf{5.39} \\
Offline-optimal (static) & 5.92 & 5.39 \\
\bottomrule
\end{tabular}

\par\medskip
\textbf{\textsc{Live vLLM: prefix caching ON vs.\ OFF}}\par\smallskip
  \begin{tabular}{lrrr}
\toprule
Metric & Caching ON & Caching OFF & Gain \\
\midrule
TTFT p50 (s) & 0.267 & 0.961 & +72\% \\
TTFT mean (s) & 0.275 & 1.086 & +75\% \\
Latency mean (s) & 2.984 & 11.293 & +74\% \\
Throughput (tok/s) & 678.4 & 180.8 & +275\% \\
Prefix-cache hit rate & 0.947 & 0.000 & --- \\
\bottomrule
\end{tabular}

\par\medskip
\textbf{\textsc{GPU-calibrated eviction policies}}\par\smallskip
  \begin{tabular}{lrr}
\toprule
Policy & TTFT saved, stationary (ms) & TTFT saved, post-shift (ms) \\
\midrule
LRU (engine default) & 12.3 & 12.9 \\
LFU & 24.1 & 8.8 \\
\textbf{Adaptive (ours)} & \textbf{24.0} & \textbf{23.6} \\
Offline-optimal & 20.4 & 12.7 \\
\bottomrule
\end{tabular}

\caption{Caching and systems detail. Mooncake: prefix-cache management on the real
tool-agent production trace of $20$k requests, measured in blocks reused per request.
Adaptive matches the offline optimum and beats the engine-default LRU, most at tight
budgets. Live vLLM: automatic prefix caching on against off, reporting TTFT, latency,
throughput, and prefix-cache hit rate on a shared-prefix workload with the built-in LRU
eviction. Eviction policies: realized TTFT saved in milliseconds on the GPU-calibrated
shared-prefix workload, stationary and post-shift. Adaptive beats the engine-default LRU
throughout and dominates after the shift.}
\label{tab:mooncake}\label{tab:vllm}\label{tab:vllmpolicy}
\par\smallskip
\textbf{\textsc{LongBench: context trimming}}\par\smallskip
\emph{Positional truncation}\par\smallskip
  \begin{tabular}{lrr}
\toprule
Policy & QA-F1 & Avg.\ input tokens \\
\midrule
Full context (keep 1.00) & 0.440 & 7239 \\
Best fixed trim & 0.440 & 7239 \\
\textbf{Adaptive per-task (ours)} & \textbf{0.362} & \textbf{3544} \\
Per-item oracle & 0.542 & 2810 \\
\bottomrule
\end{tabular}

\par\medskip
\emph{Retrieval-scored chunk selection (pre-registered follow-up)}\par\smallskip
  \begin{tabular}{lrr}
\toprule
Policy & QA-F1 & Avg.\ input tokens \\
\midrule
Full context (keep 1.00) & 0.440 & 7239 \\
Best fixed trim & 0.440 & 7239 \\
\textbf{Adaptive per-task (ours)} & \textbf{0.407} & \textbf{2701} \\
Per-item oracle & 0.563 & 2239 \\
\bottomrule
\end{tabular}

\par\medskip
\textbf{\textsc{BBH: few-shot trimming}}\par\smallskip
  \begin{tabular}{lrr}
\toprule
Policy & Accuracy & Avg.\ input tokens \\
\midrule
Full 3-shot (verbose) & 0.346 & 467 \\
Best fixed trim & 0.356 & 332 \\
\textbf{Adaptive per-task (ours)} & \textbf{0.368} & \textbf{291} \\
Per-item oracle & 0.566 & 208 \\
\bottomrule
\end{tabular}

\caption{Prompt trimming on real benchmarks: adaptive per-task selection vs.\ fixed trims and
the per-item oracle. LongBench (QA-F1 vs.\ input tokens): the top block trims by
positional truncation. The middle block is the pre-registered follow-up with the identical setup but retrieval-scored
chunk selection at the same word budgets. The better primitive lifts every fixed trim at
matched tokens ($+0.10$ to $+0.13$ F1 at tight budgets) and moves the adaptive operating
point to $0.407$ F1 at $2{,}701$ tokens. BBH exact-match vs.\ input tokens shows the adaptive
per-task few-shot selection.}
\label{tab:longbenchtrim}\label{tab:longbenchretrieval}\label{tab:bbhtrim}
\end{table}

\subsection{Prompt trimming}

Trimming spans three regimes: little headroom on MMLU, a real quality/token tradeoff on
LongBench, and strict dominance on BBH.

\clearpage
\section{Future Work}
\label{sec:future}

The structure our method exploits is not specific to trees. Two ingredients drive everything
in \autoref{sec:theory}: a metric under which the objective is smooth, and an aggregation map
whose value on a region is a (weighted) average of the points it contains. The tree supplies
the LCA distance together with the subtree average, but
the optimistic engine, the data-driven smoothness certificate, and the discontinuity-guided
sampling all carry over to any index set equipped with these two ingredients.

\emph{General metric spaces.} Replacing the LCA distance by an arbitrary metric and the dyadic
tree by a hierarchical cover, as in zooming and $\mathcal{X}$-armed bandits over metric
spaces, preserves the near-optimality-dimension analysis. The certificate of
\autoref{sec:smoothness} bounds the within-cell maximum-minus-mean from samples regardless of
how the cells are formed, so it transfers unchanged.

\emph{Spheres and directional embeddings.} The most relevant high-dimensional instance for the
language-model applications is the unit sphere under the angular (cosine) metric, since prompt
and model embeddings live there. Cells become spherical caps, the aggregation map averages
over a cap, and tree-Lipschitzness becomes Lipschitzness in geodesic distance. This would
replace the token-prefix tree of \autoref{sec:setup}, which is sparse and brittle for routing,
with an embedding-induced hierarchy on the sphere, and connects the bandit view developed here
to angular reconstruction of a reward defined over the sphere.

\emph{DAGs and weighted aggregation.} When a region's value is any linear functional of its
underlying points (a weighted subset average rather than a uniform one), the noise-deconvolved
moment-generating certificate is unchanged, so the framework extends from trees to directed
acyclic and lattice-structured index sets.

\emph{Certified sparse attention for long-context RL.} A further application targets a failure
mode recently identified in reinforcement learning of long-context language
models~\citep{sparseattnrl2026}: applying static or heuristic sparse attention during on-policy
rollouts injects an approximation error that biases the token likelihoods, and this error
compounds over the autoregressive trajectory into an actor--policy distribution mismatch large
enough to inflate the PPO/GRPO importance ratio and collapse training. The pathology is not
sparsity per se but \emph{uncontrolled, inconsistent per-step error}. Our framework suggests a
route to controlling it: cast the choice of which KV-cache blocks to attend to at each step as a
hierarchical bandit over the context tree (leaves are tokens, internal nodes are block
summaries), whose arm value is a block's contribution to the attention output and whose reward is
the noisy downstream return. Because the value is now measured through a stochastic reward rather
than computed exactly, the full optimistic apparatus of \autoref{sec:algorithm}--\ref{sec:theory}
applies, not merely the multi-fidelity certificate. Two smoothness axes structure the problem:
the \emph{spatial} tree-Lipschitzness over the context hierarchy studied here --- with the $K$
violations being exactly the isolated ``needle'' blocks that produce error spikes --- and a
\emph{temporal} assumption that block relevance drifts smoothly across generation steps, which
would bound how much re-exploration each step requires. The central open result is a bridge
theorem: a certified per-step bound on the selection error (of the form $L\rho^{\ell}+K\cdot D$
supplied by \autoref{sec:smoothness}) implies a bounded total-variation gap between the sparse and
dense policies, hence a clip-safe importance ratio and provably stable training --- a
mathematically grounded alternative to the empirical distillation used to patch this instability
today. We leave the temporal-smoothness lemma and the importance-ratio bound to future work.

\end{document}